\documentclass{tlp}

\usepackage{amsmath}
\usepackage{graphicx}
\usepackage{multirow}
\usepackage{pdflscape}

\usepackage{algorithm}
\usepackage{algorithmic}
\usepackage{amssymb}

\newtheorem{definition}{Definition}

\newtheorem{theorem}{Theorem}
\newtheorem{proposition}{Proposition}

\newtheorem{scenario}{Scenario}

\usepackage{todonotes}
\usepackage{listings}
\usepackage{fvextra} 

\definecolor{dkgreen}{rgb}{0,0.6,0}
\definecolor{gray}{rgb}{0.5,0.5,0.5}
\definecolor{mauve}{rgb}{0.58,0,0.82}

\usepackage{enumitem}
\setlist[itemize]{leftmargin=2.0em}
\setlist[enumerate]{leftmargin=2.0em}

\DefineVerbatimEnvironment{Code}{Verbatim}{
    fontsize=\small,             
    formatcom=\ttfamily,         
    commandchars=\\\{\},         
    showspaces=false,            
    showtabs=false,              
    obeytabs=true,               
    tabsize=4                    
}

\usepackage{caption}
\usepackage{subcaption}

\usepackage[style=numeric]{biblatex}
\begin{document}

\lefttitle{Human Emotion Verification}

\jnlPage{1}{8}
\jnlDoiYr{2021}
\doival{10.1017/xxxxx}



\title[Human Emotion Verification]{Human Emotion Verification by Action Languages via Answer Set Programming \footnote{This paper is an extended version of \cite{brannstrom2022emotional} and \cite{brannstrom2021modelling}, extended with a generalized formalism, introducing methods for representing, analyzing and comparing psychological theories in terms of action languages.}}

\begin{authgrp}
\author{\gn{Andreas} \sn{Brännström} }
\affiliation{Umeå University\\Department of Computing Science}
\author{\gn{Juan Carlos} \sn{Nieves} }
\affiliation{Umeå University\\Department of Computing Science}
\end{authgrp}



\maketitle

\begin{abstract}

In this paper, we introduce the action language C-MT (Mind Transition Language). It is built on top of answer set programming (ASP) and transition systems to represent how human mental states evolve in response to sequences of observable actions. Drawing on well-established psychological theories, such as the Appraisal Theory of Emotion, we formalize mental states—such as emotions—as multi-dimensional configurations. With the objective to address the need for controlled agent behaviors and to restrict unwanted mental side-effects of actions, we extend the language with a novel causal rule, forbids to cause, along with expressions specialized for mental state dynamics, which enables the modeling of principles for valid transitions between mental states. These principles of mental change are translated into transition constraints, and properties of invariance, which are rigorously evaluated using transition systems in terms of so-called trajectories. This enables controlled reasoning about the dynamic evolution of human mental states. Furthermore, the framework supports the comparison of different dynamics of change by analyzing trajectories that adhere to different psychological principles. We apply the action language to design models for emotion verification. Under consideration in Theory and Practice of Logic Programming (TPLP).

\end{abstract}


\section{Introduction}
\label{sec:Introduction}

\noindent Interactive and intelligent systems are increasingly being designed to display human-like mental capabilities \cite{adikari2022empathic,milcent2022using,morris2018towards,martinengo2019suicide,burger2020technological}.
For instance, in the area of health and wellbeing, software assistants are being developed to display complex human traits, such as empathy and sympathy \cite{morris2018towards}, to deliver emotionally charged actions \cite{adikari2022empathic} or to provoke empathic responses from users \cite{milcent2022using}. Some of these systems are deployed in society for, e.g., depression support, therapy and behavior-change interventions \cite{martinengo2019suicide, burger2020technological}.
In such applications, that in various ways deal with manipulation of human mental states, such as emotions, 
ensuring reliable system behavior is crucial
\cite{muise2019planning}. A system may need to constrain and plan its interactions in order to anticipate and reduce unwanted influence as a result of its behavior. Moreover, effective methods for verifying that such side-effects are avoided are essential. This requires dynamic models of the human mind that capture the causality of mental states and support formal reasoning about the consequences of interaction strategies.

Research on formal models of human behavior spans from physiology \cite{fass2009rationale} to social practices \cite{erdogan2025toma}, with extensive work on mental states \cite{pereira2007formal,jiang2007ebdi,ong2019computational}. This includes modeling emotional responses, goal and plan recognition \cite{ramirez2009plan,shvo2020active}, and extensions of planning to epistemic, empathetic, and cognitive settings \cite{bolander2011epistemic,shvo2019towards,lorini2022cognitive}. Logics of attitudes and emotions further enrich epistemic logic with belief strength, appraisal, and counterfactual reasoning \cite{lorini2021qualitative,adam2009logical,lorini2011logic}. Despite this progress, few approaches formalize principled constraints on how mental states evolve. While psychological theories such as emotion regulation \cite{ortner2018roles,tamir2008hedonic} describe such patterns, formal methods rarely capture them as transition constraints or invariants \cite{hansen2003algorithms}. A more structured approach would be to model mental states as multi-variable configurations governed by constraints on how they evolve over sequences of state changes, enabling verification of unwanted influence in their long-term evolution.

In this paper, we introduce the action language ${\cal C}_{MT}$ (Mind Transition Language), designed to model changes in human mental states in response to observable actions. Through its translation into Answer Set Programming (ASP) \cite{brewka2011answer}, the language supports the generation and verification of action trajectories—sequences of actions and resulting states—while enforcing rules that forbid undesirable mental state transitions. To illustrate its use, we apply ${\cal C}_{MT}$ to psychological theories: Roseman’s Appraisal Theory of Emotion \cite{roseman1996appraisal} provides a characterization of mental states—here, emotion states—as multi-dimensional configurations, and theories of emotion regulation \cite{zaki2020integrating,tamir2007business} serve as an example of how principles of forbidden change can be specified. Although the language itself is general and can be instantiated with alternative theoretical accounts, in this work we present a case study based on these particular characterizations. The central aim of the framework is to support controlled agent behavior that minimizes unwanted mental side effects by enforcing such principles of forbidden change.

\begin{figure}
    \centering
    \includegraphics[width=1.0\linewidth]{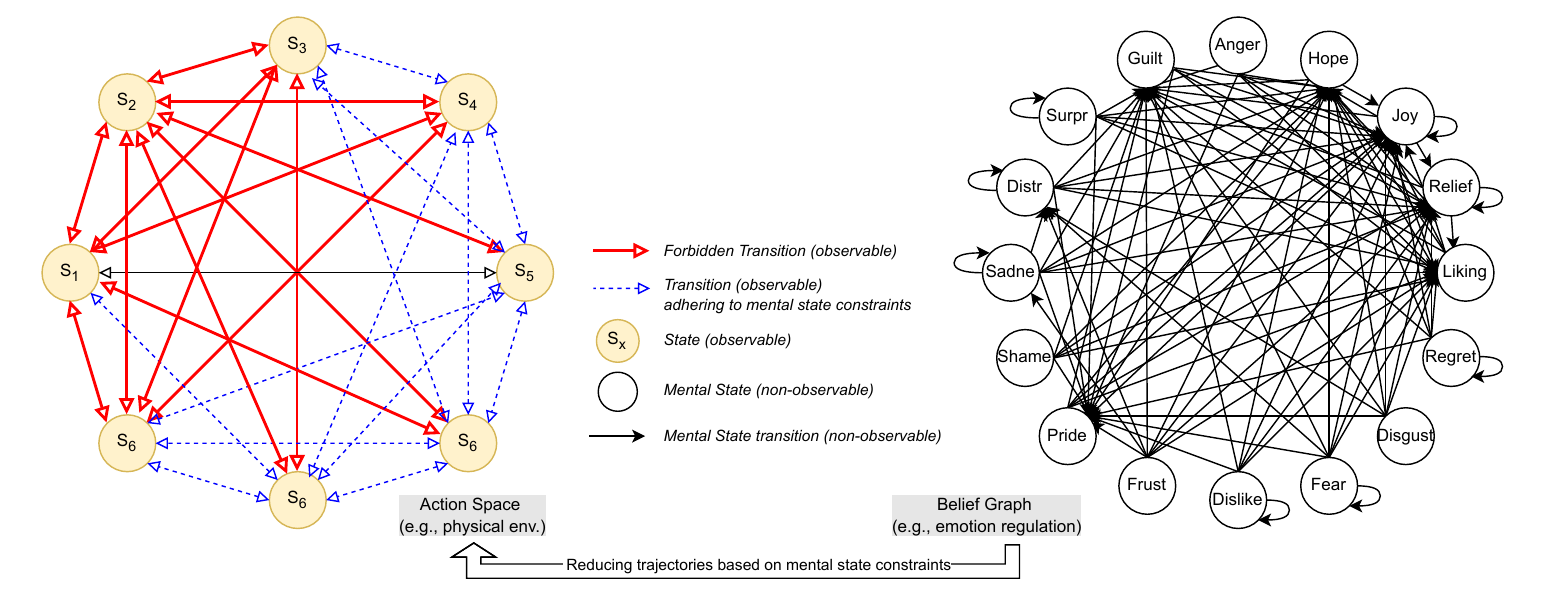}
    \caption{ In ${\cal C}_{MT}$, available sequences of actions and states (trajectories) in the ``physical'' state space are constrained by the actions' influence on the ``mental'' state space. }
    \label{fig:Trajectory-space}
\end{figure}

While action languages serve as intuitive specifications for dynamic reasoning processes, they can be characterized as Answer Set Programs. As a result, ${\cal C}_{MT}$ provides a logic programming foundation for knowledge elicitation and engineering of mental state dynamics. Emotional reasoning is inherently complex due to the exponential space of emotion states, particularly when accounting for their dynamics. This complexity poses a computational challenge, demanding solutions to combinatorial problems. ASP is well suited to address these challenges: it can express NP-search problems that are solvable by a nondeterministic Turing machine in polynomial time, with solutions encoded as answer sets \cite{brewka2011answer}. Thus, logic programming, and ASP in particular, offers an effective method for implementing and managing emotional reasoning. With this motivation, we introduce the framework’s syntax and semantics, and establish its link to ASP.

This further enables ${\cal C}_{MT}$ to capture dynamic properties such as \emph{invariance} \cite{hansen2003algorithms}, understood in transition systems as a property that must hold in all states along every possible trajectory. In our setting, invariance specifies properties over mental state changes such that, when satisfied, the evolution of states follows a corresponding psychological principle of change, which can be rigorously analyzed through formal verification. Ultimately, the framework supports reasoning about the evolution of human mental states while simultaneously constraining actions in the ``physical'' environment (see Figure \ref{fig:Trajectory-space}).

From this basis, we present the following key contributions:

\begin{itemize}
	\item We introduce the action language ${\cal C}_{MT}$, serving as a foundational platform for capturing properties relevant to mental state reasoning.
    \item We present the framework's syntax and semantics, and its practical implementation in Answer Set Programming, evaluated formally and empirically.
    \item We present characterizations of ${\cal C}_{MT}$ in the setting of emotional reasoning, capturing psychological theories such as Appraisal theory of Emotion \cite{roseman1996appraisal}, Hedonic Emotion Regulation \cite{zaki2020integrating} and Utilitarian Emotion Regulation \cite{tamir2007business}.
    \item We present how ${\cal C}_{MT}$ can be applied as a method for representing, verifying and comparing psychological theories in terms of action trajectories.
\end{itemize}

The remainder of this paper is structured as follows. 
Section~\ref{sec:Motivational-Background} presents the motivational background. 
Section~\ref{sec:Related Work} reviews related work. 
Section~\ref{sec:Formal Framework} introduces the formal framework. 
Section~\ref{sec:Case Study Emotion} presents a case study in emotional reasoning. 
Section~\ref{sec:Formal Analysis} provides a formal analysis of the framework. 
Section~\ref{sec:Evaluation} reports on the experimental evaluation. 
Section~\ref{sec:exmaple-human-emotion-verification} illustrates the approach through an example in human emotion verification. 
Section~\ref{sec:Discussion} offers a broader discussion of the results. 
Finally, Section~\ref{sec:Conclusion} concludes the paper and outlines directions for future work.

\section{Motivational Background}
\label{sec:Motivational-Background}

In today’s digital society, where social media and AI-based systems are deeply embedded in everyday interactions, the potential for manipulation and undue influence—whether by people or automated systems—has become a serious concern \cite{park2024ai}. A real example \cite{singleton2023chatbot} is the case of an individual who was sentenced to nine years for attempting to assassinate Queen Elizabeth II, after exchanging thousands of messages with a chatbot that appeared to encourage his violent intentions. Such incidents underscore the urgent need for methods that can verify, explain, and prevent unsafe forms of influence in human–AI interaction, where emotional and motivational dynamics play a central role. 

An underlying issue regards limitations in user modeling and personalization. State-of-the-art user modeling in interactive systems commonly use statistical methods to characterize users, based on, e.g., user interaction history and demographic information \cite{burger2020technological,masthoff2014preface,rabbi2015mybehavior}. In this way, systems make approximations of users, based on snapshots of human behavior. 
However, a user’s physical and mental state may change over time during an interaction, meaning that observations from a previous moment (e.g., the last action or the previous day) may no longer accurately reflect the user’s current state. 
Consequently, in a human-system interaction, a model based on a snapshot of the human may lead to uncontrolled and unwanted agent behavior. 
In order for a software agent to execute suitable actions in interactions with humans, the agent must consider the mental states of its human interlocutors in its internal reasoning and decision-making. This ability, known as Theory of Mind (ToM) \cite{frith2005theory}, involves inferring another agent’s beliefs, emotions, motivations, goals, and intentions. To develop systems that effectively compute the ToM of their users, it is essential to make these systems aware of the human mental properties that may change during interactions. This involves creating models in terms of states and transitions between states.

The proposed framework rests on two fundamental pillars: psychological theories of emotion and non-monotonic reasoning via ASP. The psychological theories, such as the Appraisal Theory of Emotion, provide a structured and scientifically grounded basis for modeling the complex dynamics of human emotions. 
On the other hand, the non-monotonic reasoning approach enabled by ASP complements these psychological foundations with a computational mechanism to handle the inherent variability in emotional reasoning. 

Action languages provide intuitive means of specifying dynamic reasoning processes, making them highly effective for knowledge elicitation and engineering. In our framework, the action language ${\cal C}_{MT}$ is used to describe actions and their effects on mental states, as well as the constraints governing these effects. By linking the action language to ASP, we enable the implementation of these specified dynamics in a computationally efficient manner. ASP is particularly suited for this purpose due to its ability to handle complex combinatorial problems, such as those inherent in emotional reasoning. It allows for the expression of NP-search problems that can be solved using nondeterministic polynomial time, with solutions encoded as answer sets \cite{brewka2011answer}. 

Psychological theories offer structured frameworks to explain various aspects of human cognition, emotion, and behavior. Appraisal theories, such as the appraisal theory by Roseman (1996) \cite{roseman1996appraisal} which is modeled as an example in this work, identify classes of determinants and how combinations of them define emotions. From this, we can formalize fluents (changeable variables) and states (configurations of fluents) in our framework. In the appraisal theory by Roseman (1996), emotion states are represented through variables such as need-consistency, goal-consistency, accountability, and control potential. By structuring mental states as configurations of these factors, we can represent, analyze, and reason about human mental states and the principles governing their transitions. Hence, in the proposed framework, a mental state of a human agent is approximated through a multi-dimensional representation, where each dimension corresponds to a factor in a psychological theory. In this way, we create abstractions of mental states, such as emotions, transforming psychological theories into computational models. 

Transitioning between mental states involves adhering to principles derived from psychological theories. For example, Hedonic Emotion Regulation (HER) \cite{zaki2020integrating} focuses on increasing positive emotions and decreasing negative ones, while Utilitarian Emotion Regulation (UER) \cite{tamir2007business} emphasizes the functional outcomes of emotion states. In our framework, these principles are translated into constraints that define valid mental state changes. By implementing these constraints in ASP, we ensure properties such as \emph{invariance} \cite{hansen2003algorithms}, preventing entry into unwanted mental states.

A key distinction between our approach and traditional statistical methods lies in the deterministic nature of the latter versus the non-deterministic nature of ASP in our framework. Statistical methods typically collapse a decision into a single point, which can lead to a loss of information and a lack of flexibility in reasoning processes. In contrast, ASP's non-determinism yields multiple answer sets, offering a range of possible solutions that can be filtered and analyzed according to the introduced constraints and principles of change grounded in psychological theories. This approach not only preserves the richness of potential outcomes but also allows for a more nuanced and comprehensive understanding of the dynamics of human mental states.

The practical relevance of this approach is clear in domains where mental states are central to interaction. Examples include conversational agents for depression support, therapy, and behavior change \cite{adikari2022empathic,milcent2022using,morris2018towards,martinengo2019suicide,burger2020technological}. 
In such settings, modeling and constraining mental state transitions is essential to delivering coherent and safe interactions.

\section{Related Work}
\label{sec:Related Work}

Human behavior has been formally analyzed at multiple levels and abstractions, ranging from integrative physiology \cite{fass2009rationale} at a low level to formalizations of social practices \cite{erdogan2025toma} at a high level. On the level of mental states \cite{erdogan2025toma, pereira2007formal, jiang2007ebdi, belkaid2014logical, ong2019computational}, research has modeled mental-state context, such as simulating emotional behavior \cite{ong2019computational} or modeling expected human responses to affective states \cite{jiang2007ebdi}. 
Moreover, a diverse body of research has explored the dynamics of mental states \cite{ramirez2009plan, shvo2019towards, shvo2020active, rao1995bdi, keren2014goal}. Plan recognition as planning, introduced by Ramirez and Geffner \cite{ramirez2009plan}, applies classical AI planning techniques to infer the goals and plans of agents based on observed actions. 
Active Goal Recognition (AGR) \cite{shvo2020active} extends goal recognition using contingent planning and landmark-based hypothesis elimination, enabling an active observer to sense, reason, and act.
Furthermore, in Epistemic planning \cite{bolander2011epistemic}, a generalization of classical planning, involves agents specifying goals that include the epistemic state, such as beliefs, of other agents. Empathetic planning \cite{shvo2019towards} formalizes empathy as reasoning about another agent’s preferences, utilizing multi-agent epistemic planning where an agent models the beliefs and goals of a human agent. Recent extensions of epistemic planning to cognitive planning formalizes a method for influencing the cognitive state of the target agent \cite{lorini2022cognitive, davila2021simple}.

While modal logic underpins many approaches to reasoning about knowledge and belief, it typically represents mental states as sets of possible worlds, often reducing complex phenomena such as ``emotions'' to atomic propositions (e.g., ``Joy'' or ``Sadness'') rather than multi-variable configurations with structured interdependencies. This limits expressivity in capturing the causes, constraints, and consequences of mental-state transitions over time. To address this, various logics of mental attitudes and emotions \cite{lorini2021qualitative, adam2009logical, lorini2011logic, dastani2012logic, steunebrink2012formal} have been developed. These works integrate epistemic logic with additional structures, such as plausibility orderings for belief strength \cite{lorini2021qualitative}, logic-based appraisal models for emotions \cite{adam2009logical}, and STIT (Seeing-To-It-That) logic for counterfactual reasoning about emotions such as regret and rejoicing \cite{lorini2011logic}. A notable approach is the TOMA framework \cite{erdogan2025toma}, which integrates epistemic logic to model and update beliefs while introducing higher-order abstractions based on lower-order beliefs. By employing modal operators for belief and knowledge, the system facilitates structured reasoning about trust, social roles, and norms. 

In the setting of mental state representation and reasoning, two related areas of research regard: 1) logics of mental attitudes and emotion \cite{lorini2021qualitative,adam2009logical,lorini2011logic,dastani2012logic,steunebrink2012formal}, and 2) as previously mentioned, epistemic planning \cite{bolander2011epistemic} extended to cognitive planning \cite{lorini2022cognitive, davila2021simple}. 
Logics of mental attitudes and emotion aim to formalize the relationships between epistemic and motivational attitudes of human and artificial agents, as well as the influence of mental attitudes on emotions. Some related works in this line of research include the logical formalization of OCC theory of emotions \cite{adam2009logical}, the formalization of counterfactual emotions \cite{lorini2011logic}, the representation of emotion intensity and coping strategies \cite{dastani2012logic}, the modeling of emotion triggers \cite{steunebrink2012formal}, and the logical theory of epistemic and motivational attitudes and their dynamics \cite{lorini2021qualitative}. Nevertheless, the principles constraining this causality and its potential side effects are not considered. 
In contrast to previous approaches to modeling and reasoning about mental states, the proposed ${\cal C}_{MT}$ action language deals with a multi-dimensional representation of mental states and the constraints for modeling principle-based transitions between them.

In the area of affective agents and computational theory of mind, agent models have been developed to reason about emotion and behavior \cite{qiu2022towards,yongsatianchot2021computational,ong2019computational,jara2019theory,si2010modeling}. For instance, agents based on Partially Observable Markov Decision Processes (POMDP) \cite{yongsatianchot2021computational} have been used to model emotion, showing potential in simulating human behavior. Nevertheless, the state space is solely in the physical environment. Hence, they have lacked to capture mental-state dynamics to reason about causes for mental-states and mental transitions. 
Inverse Reinforcement Learning (IRL) has been proposed to model Theory of Mind \cite{ong2019computational}, for predicting peoples' actions and inferring mental states through policy reconstruction. However, a limitation with their approach is that it assumes identical, and rational, decision-making due to its inability to capture the variability in human reasoning, neglecting individual differences and emotional influences that may significantly impact human behavior. To effectively model the dynamics of the human mind, capturing individual and contextual factors, a low level, multi-dimensional approach is required.
A study on socially-aware agents \cite{qiu2022towards} proposes a hybrid mental state parser that extracts information from dialogue and event observations to maintain a graphical representation of an agent's mind. The nodes represent agents, their personas, objects, and descriptions of the setting. The edges between these nodes depict the state of mind of the agents, capturing how their beliefs and mental states change as the setting evolves. While their representation can represent mental state changes, it lacks the ability to represent constraints of changes for analyzing their adherence to specific principles of mental change. Also, the graphical representation is on a high level, in contrast to the multidimensional representation proposed in the current paper.

Let us also mention that there is a range of approaches in the setting of BDI (Belief-Desire-Intention) agents \cite{pereira2007formal,jiang2007ebdi,jones2009personality,sanchez2019designing,sanchez2019abc}, integrating affective states into traditional BDI models, such as to simulate expected emotional responses \cite{pereira2007formal} or provide architectures for embedding emotion-driven reasoning \cite{jiang2007ebdi}. While these works have considered affective states, such as emotion, integrated the BDI model, challenges persist in implementing affective constraints, such as emotion regulation \cite{sanchez2019designing}.

The related paradigm of Human-Aware Planning (HAP) \cite{leonetti2019adaptive,chakraborti2018human,ahrndt2014human} focuses on planning in the state-space of the physical environment, often in shared contexts where both autonomous agents and humans act. The planner typically reasons about the human’s physical actions and coexists with them in domains such as robot navigation or collaborative manipulation \cite{chakraborti2018human}. However, for a rational software agent to anticipate how its actions or events may influence a human’s mental state, it must reason in a space that goes beyond the physical. This requires a planning model that explicitly includes the state-space of the human mind: mental states, allowable transitions between them, and the actions that may trigger such transitions.

From this background, we highlight that there is a lack of formal treatment on principled constraints governing how mental states evolve over an interaction. While psychological theories, such as emotion regulation \cite{ortner2018roles, tamir2008hedonic}, describe patterns of mental-state change, formal models rarely integrate such principles, e.g., as transition constraints or properties of invariance \cite{hansen2003algorithms}, which can be rigorously evaluated using formal methods.
By enforcing such constraints, a system can verify that mental-state trajectories, in contrast to isolated transitions, adhere to predefined principles while enabling the detection of deviations or violations. 
A more structured approach represents mental states as multi-variable configurations governed by trajectory-level constraints that regulate their long-term evolution. We refer to such a framework for defining valid mental states and transitions as a Belief Graph (BG), formally introduced and analyzed in this paper.

In the following sections, we introduce the formal framework that integrates these concepts, providing a structured approach to modeling and managing the dynamics of human mental states.

\section{Formal Framework}
\label{sec:Formal Framework}

This section introduces the syntax and semantics of the proposed formal framework and the action language extension ${\cal C}_{MT}$, which builds on ASP-based action reasoning. Although ${\cal C}_{MT}$ is implementation-agnostic and can, in principle, be realized using different computational frameworks, we view it as a high-level formal language designed specifically for realization in Answer Set Programming (ASP). From this perspective, we construct expressions in ${\cal C}_{MT}$ to align naturally with their ASP encodings. The framework offers flexibility to specialize for specific mental state domains, such as emotions, aligning with established psychological theories, such as the Appraisal theory of Emotion by Roseman (1996) \cite{roseman1996appraisal}. By incorporating principles of mental change, sets of transition constraints are formalized and implemented in terms of integrity constraints in answer set programs.

In the proposed action language, similar to previous action languages, \emph{fluents} represent properties that can change over time. These can describe various aspects of human-agent interactions, including non-observable aspects such as psychological attributes and observable aspects of agents or the environment. The value of a fluent at any given time depends on how it is affected by so-called \emph{actions} or indirectly by other fluents. 

We build on the action language ${\cal C}_{TAID}$ \cite{dworschak2008mbn}, originally designed for modeling biological systems. ${\cal C}_{TAID}$ includes features such as \emph{allowance}, \emph{inhibition}, and \emph{triggers}, which are also relevant for modeling the dynamics of mental states. The \emph{triggers} causal law accounts for interactions based on reactions, making it particularly useful for capturing how mental states may change as indirect effects of actions or external events. \emph{Allowance} rules specify that an action can occur under certain conditions but is not mandatory, while \emph{inhibition} rules prevent an action from occurring in certain contexts. These mechanisms are especially important in modeling mental states, where dependencies between cognitive and emotional factors may be partially known or not explicitly modeled. In such cases, allowance and inhibition rules provide a flexible way to account for uncertainties and exceptions in mental state reasoning.
These rules are relevant in the context of mental states where we have partial knowledge about the dependencies and reasons behind interactions in the mind. In situations where dependencies are partially known or not explicitly modeled, such as some conditions of the environment, allowance and inhibition rules provide a flexible way to handle uncertainties and exceptions.

We introduce the action language ${\cal C}_{MT}$ (Mind Transition Language), serving as a foundational platform, capturing domain-independent properties and constraints relevant to mental state reasoning. Along with the action language, we introduce some ``syntactic sugar'' to facilitate expressions about mental state dynamics. This includes ``influences mental fluent'' (similar to ``causes'' in ${\cal C}_{TAID}$), handling actions or events that may influence a change in a mental state, and the rules ``facilitates'' and ``contravenes'' (similar to ``allowance'' and ``inhibition'' rules in ${\cal C}_{TAID}$), which regulate actions' execution in an initial state. However, we specialize these rules to particularly concern \emph{human actions} regulated by \emph{mental fluents}. This is motivated by emotion theories, suggesting that ``emotions have distinctive goals and action tendencies'' \cite{roseman1994phenomenology}.

Moreover, ${\cal C}_{MT}$ introduces abstractions that we call \textbf{mental states}, which are defined by sets of fluents, along with mental state transition constraints, called \textbf{forbids to cause} rules, that specify relationships among mental fluents from one state to the next. Notably, such expressions enabling direct restriction of fluents between states, independent of actions, have not been explicitly incorporated in previous action languages like ${\cal C}_{TAID}$.

Consequently, the proposed framework for mental-state reasoning is comprised of two components: 1) The action language ${\cal C}_{MT}$, which defines actions in the environment/interaction that trigger changes in mental states, and 2) a set of constraints that precisely define valid transitions between mental states, derived from psychological principles. These constraints characterize a so-called Belief Graph (BG). In ASP, a BG is encoded as sets of integrity constraints, restricting particular fluent changes in transitions between mental states. Mental state dynamics are linked to actions in the environment by considering their effects on fluents in the mental state abstractions. In this way, a BG filters the potential trajectories resulting from the action language, based on their effects on mental states (see Figure \ref{fig:Mind-aware-planning-framework} for a conceptualization of the framework). 

Intuitively, ${\cal C}_{MT}$ extends ${\cal C}_{TAID}$ as follows (formalized later in this section):
\begin{itemize}
    \item To begin with, it refines the action and fluent alphabet by distinguishing between \emph{observable} environmental fluents (denoted ${\bf F}^E$) and actions (denoted ${\bf A}^E$) and \emph{non-observable} mental fluents (denoted ${\bf F}^H$) and human actions (denoted ${\bf A}^H$).
    
    \item It introduces specialized constructs such as \texttt{influences}, \texttt{facilitates}, and \texttt{contravenes}, which replicate ${\cal C}_{TAID}$’s \texttt{causes}, \texttt{allows}, and \texttt{inhibits} rules but are designed specifically to model how mental states regulate human behavior. This is made explicit in the ASP encoding, where these rules specifically are linked to mental fluents and human actions. While these constructs do not provide a substantial formal contribution, they support knowledge elicitation regarding modeling of mental dynamics, in line with one of the core purposes behind action languages. 

    \item  It incorporates a novel constraint rule, \[
(g_1^h, \dots, g_m^h \; \mathbf{forbids~to~cause} \; f_1^h, \dots, f_n^h)
\]
specifies that if the mental fluents \( g_1^h, \dots, g_m^h \in {\bf F}^H \) hold at time \( t \), then the mental fluents \( f_1^h, \dots, f_n^h \in {\bf F}^H \) are forbidden from holding at time \( t+1 \). This construct enables the specification of state transitions that are invalid regardless of any action occurrence—something not expressible in ${\cal C}_{TAID}$.

    \item A contribution of the language, utilizing the new \textbf{forbids to cause} rule, is the overall methodology for formalizing psychological theories in terms of dynamic computational models: abstractions of sets of fluents called mental-states and valid transitions between them. A set of these forbidding rules, on top of a set of dynamic causal laws, defines a so-called \emph{Belief Graph} (BG), which captures valid mental state transitions and filters action trajectories with constraints based on principles from psychological theories.
\end{itemize}

These additions enhance ${\cal C}_{MT}$ by enabling it to model complex dynamics of mental states and verify system behavior in accordance with psychological principles. The novel \textbf{forbids to cause} rule is essential in mental state modeling for two key reasons: (1) actions may directly induce undesirable changes in mental fluents, and (2) such changes may indirectly trigger unwanted ramification effects on other mental fluents. Consequently, this rule provides a formal mechanism to restrict actions that \emph{directly} or \emph{indirectly} lead to adverse mental outcomes.

 \begin{figure}[ht!]
 \centering
  \includegraphics[width=0.8\textwidth]{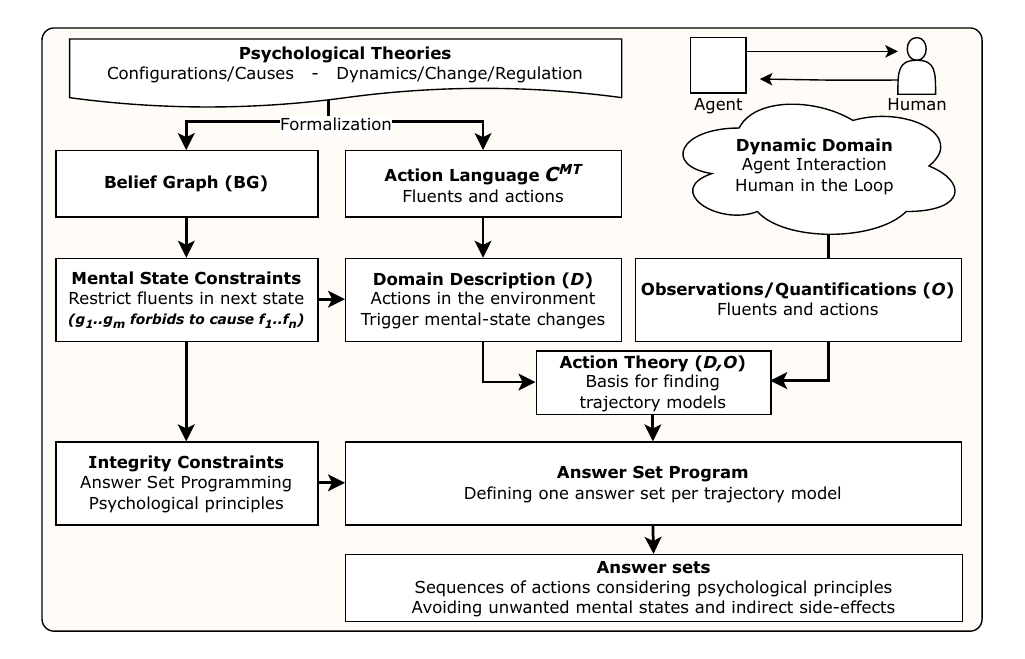}
	\caption{Conceptual Framework}
	\label{fig:Mind-aware-planning-framework}
\end{figure}

\subsection{Belief Graph (BG)}

A specialized BG is a set of propositional atoms for valid transitions between mental states. The mental states within the BG represent configurations of factors that contribute to specific mental states. These factors, known as mental fluents, are determined by psychological theories and encompass a range of possible values, thus defining the potential states within the BG.
The mental fluents that constitute the states of the BG are defined as follows:

\begin{definition}[Mental fluent]\label{def:mf}
Let $C = \{c_1, \dots, c_n\} $ be a set of symbols denoting psychological classes, and let $V = \{V_{c_1}, \dots,V_{c_n} \}$ be a set of total ordered sets of constants denoting psychological values for each class of $C$.
A mental fluent is a ground atom $f(c, v)$ of arity 2 such that $c \in C$, $v \in V_c$. 
\end{definition}

A mental state space $S$ is defined as the set of all possible combinations of mental fluents.

\begin{definition}[Mental state space]\label{def:state_space}
Let $C = \{c_1, \dots, c_n\}$ be a set of psychological classes, and let $V = \{V_{c_1}, \dots, V_{c_n}\}$ be a set where each $V_{c_i} (1 \leq i \leq n)$ is the set of possible values for the class $c_i$. The mental state space $S$ is defined as:
\[
S = \bigcup_{ v_1 \in V_{c_1},  v_2 \in V_{c_2},\dots, v_n \in V_{c_n}} \{ f(c_1, v_1), f(c_2, v_2), \ldots, f(c_n, v_n) \} 
\]
where each set $\{ f(c_1, v_1), f(c_2, v_2), \ldots, f(c_n, v_n) \}$ denotes a unique combination of mental fluents, called a mental state.
\end{definition}

We now proceed by defining a BG with states $S$ and directed edges $E \subseteq S \times S$. 
These edges between states capture principles of mental change suggested by psychological theories, providing rules for allowed mental state transitions.

\begin{definition}[Belief Graph]
\label{def: Belief Graph}

A Belief Graph is a directed graph $BG = (S,E)$ where $S$ is a finite set of mental states and $E \subseteq S \times S$ is a set of directed edges representing valid transitions between mental states. 
\end{definition}

A BG needs to be characterized for each specific application since different approaches to mental influence may be applicable. Consequently, to support controlled system behavior and principle-based assessment, it is essential to establish reasoning principles of ``monotonicity'' governing valid transitions between mental states. These principles can be motivated from various sources, such as psychological theories or insights from human experts. In the setting of reasoning about transitions between emotions, there are different principles and theories that have contrasting views on emotional change. For instance, principles of \emph{hedonic emotion regulation} \cite{zaki2020integrating} aim to increase positive emotions and decrease negative emotions. Another set of principles regards \emph{utilitarian emotion regulation} \cite{tamir2007business}, which aim to increase emotions that provide utility, such as control. Hence, the BG is agnostic to the choice of psychological theories. We show how this flexibility enables the framework to model, and compare, psychological theories through the principle-based assessment of generated trajectories.

\subsection{The ${\cal C}_{MT}$ Action Language}

${\cal C}_{MT}$ consists of a set of symbols representing actions and fluents, forming the alphabet of the action language. Given that the purpose of an action language is to provide a higher-level framework that makes modeling dynamic systems more natural and modular, several aspects of the language serve as ``syntactic sugar'' to simplify the modeling of mental states. The primary addition, beyond a specialized alphabet, is the incorporation of \textbf{forbids to cause} rules that given a set of mental fluents in the current state forbids a set of mental fluents to hold in the next state. By capturing elements of both the external environment and the human mind, the language facilitates the description and analysis of how environmental events influence human mental states and behavior, and through the new forbidding constraints, \emph{actions} as well as \emph{indirect effects} on mental fluents can be constrained. We require this addition to the language in order to model principles of mental change.

\begin{definition}[${\cal C}_{MT}$ alphabet]
Let ${\bf A}$ be a non-empty set of actions and ${\bf F}$ be a non-empty set of fluents.

\begin{itemize}
  \item ${\bf F} = {\bf F}^E \cup {\bf F}^H$ such that ${\bf F}^E$ is a non-empty set of fluent literals describing observable items in an environment and ${\bf F}^H$ is a non-empty set of fluent literals describing quantified non-observable features of mental-states of humans. ${\bf F}^E$ and ${\bf F}^H$ are pairwise disjoint.
  \item ${\bf A} = {\bf A}^E \cup {\bf A}^H$ such that ${\bf A}^E$ is a non-empty set of actions that can be performed by a software agent and ${\bf A}^H$ is non-empty set of actions that can be performed by a human agent. ${\bf A}^E$ and ${\bf A}^H$ are pairwise disjoint.
\end{itemize}
\end{definition}

Within ${\cal C}_{MT}$, a domain description defines static and dynamic causal laws for actions. These laws precisely express the expected influences exerted on mental fluents, either as direct effects of actions or as indirect causal effects.
The laws governing mental change operate by modulating the variables of a given mental state while adhering to the constraints outlined by the BG.

\begin{definition}[${\cal C}_{MT}$ domain description language]
\label{def:domain description language}
The ${\cal C}_{MT}$ domain description language $D^{MT}({\bf A}, {\bf F})$
consists of static and dynamic causal laws of the following form:

\begin{tabular}{ll}
${\cal C}_{MT}$ domain description language (extending ${\cal C}_{TAID}$):\\
$(a \; \mathbf{causes} \; f_1,\dots,f_n \; \mathbf{if} \; g_1, \dots, g_m)$ & $(1)$ \\
$(f_1,\dots,f_n \; \mathbf{if} \; g_1, \dots, g_m)$ & $(2)$ \\
$(f_1,\dots,f_n \; \mathbf{triggers} \; a)$ & $(3)$ \\
$(f_1,\dots,f_n \; \mathbf{allows} \; a)$ & $(4)$ \\
$(f_1,\dots,f_n \; \mathbf{inhibits} \; a)$ & $(5)$ \\
$(\mathbf{noconcurrency} \; a_1, \dots, a_n)$ & $(6)$ \\
$(\mathbf{default} \; g)$ & $(7)$ \\ \\
Mental-state domain description language extension:\\
$(a \; \mathbf{influences} \; f_1^h, \dots f_n^h \; \mathbf{if} \; g_1, \dots, g_m)$ & $(8)$ \\
$(g_1, \dots, g_m \; \mathbf{influences} \; f_1^h,\dots,f_n^h)$ & $(9)$ \\
$(g_1^h, \dots, g_m^h \; \mathbf{facilitates} \; a^h)$ & $(10)$ \\
$(g_1^h, \dots, g_m^h \; \mathbf{contravenes} \; a^h)$ & $(11)$ \\
$(g_1^h,\dots,g_m^h\; \mathbf{forbids~to~cause} \; f_1^h, \dots, f_n^h)$ & $(12)$ \\

& \\
\end{tabular}

\noindent where $a \in {\bf A}$, $a^h \in {\bf A}^H$, and $a_i \in {\bf A}$ ($0 \leq i \leq n$) and $f_j \in {\bf F}$, $(0 \leq j \leq n)$ and $g_j \in {\bf F}$, $(0 \leq j \leq m)$, and $f_1^h, \dots, f_n^h \in {\bf F}^H$, $g_1^h, \dots, g_m^h \in {\bf F}^H$ are mental fluents of the form $f(c,v)$, where $c$ is a psychological class and $v$ is a psychological value.
\end{definition}

We next provide the formal specification of the action language components in terms of states, rules, and transition conditions. Once this foundation is in place, we proceed to their characterization under answer set semantics, where the operational meaning of the laws is captured by logic programs and their answer sets.

BGs are expressed in terms of ${\cal C}_{MT}$ logic programs, i.e., finite sets of dynamic and causal laws on mental states. This characterization allows us to ensure controlled mental change by restricting the states and state-transitions.
The mental state, within the context of the domain description, represents an interpretation of the current state of the system.

\begin{definition}[Mental state interpretation]
\label{def: c-ae emotion state}
A state $s \in S$ of the domain description $D^{MT}(\bf A, F)$ is an interpretation over ${\bf F}$ such that

\begin{enumerate}
	\item for every static causal law $(f_1,\dots,f_n \; \mathbf{if} \; g_1, \dots g_m) \in D^{MT}(\bf A, F)$, we have $\{f_1,\dots,f_n\} \subseteq s$ whenever $\{g_1, \dots g_m\} \subseteq s$.

	\item for every static causal law $(g_1,\dots,g_m \; \mathbf{influences} \; f_1^h,\dots,f_n^h) \in D^{MT}(\bf A, F)$, we have $\{f_1^h,\dots,f_n^h\} \subseteq s$ whenever $\{g_1,\dots,g_m\} \subseteq s$, and $\{f_1^h,\dots,f_n^h\} \subseteq {\bf F}^H$.

    
\end{enumerate}

$S$ denotes all the possible states of $D^{MT}(\bf A, F)$.

\end{definition}

In this definition, a state is determined by the satisfaction of static causal laws within the domain description $D^{MT}({\bf A}, {\bf F})$. The first condition ensures that if the prerequisite mental fluents for an action are true, then the consequent mental fluents will also be true in the state. The second condition specifies that if certain mental fluents influence a particular mental fluent, then the influenced mental fluent must be true when all the influencing mental fluents are true. 

Let us define the laws of the domain description more precisely.

\begin{definition}[Domain description]

	By considering the domain description $D^{MT}(\bf A, F)$ and a state $s$, the following rules and laws apply:
    
	\begin{enumerate}

    	\item
    	An inhibition rule ($f_1,\dots,f_n \;$ \textbf{inhibits} a) is active in s, if $f_1,\dots,f_n \; \in s$, otherwise, passive. The set $A_I(s)$ is the set of actions for which there exists at least one active inhibition rule in s (as in ${\cal C}_{TAID}$ \cite{dworschak2008mbn}).
   	 
    	\item
    	A triggering rule ($f_1,\dots,f_n \;$ \textbf{triggers} a) is active in s, if $f_1,\dots,f_n \; \in s$ and all inhibition rules of action a are passive in s, otherwise, the triggering rule is passive in s. The set $A_T(s)$ is the set of actions for which there exists at least one active triggering rule in s. The set $\overline{A}_T(s)$ is the set of actions for which there exists at least one triggering rule and all triggering rules are passive in s (as in ${\cal C}_{TAID}$ \cite{dworschak2008mbn}).
   	 
    	\item
    	An allowance rule ($f_1,\dots,f_n \;$ \textbf{allows} a) is active in s, if $f_1,\dots,f_n \; \in s$ and all inhibition rules of action a are passive in s, otherwise, the allowance rule is passive in s. The set $A_A(s)$ is the set of actions for which there exists at least one active allowance rule in s. The set $\overline{A}_A(s)$ is the set of actions for which there exists at least one allowance rule and all allowance rules are passive in s (as in ${\cal C}_{TAID}$ \cite{dworschak2008mbn}).

    	\item
    	A facilitating rule ($g_1^h, \dots, g_m^h \;$ \textbf{facilitates} $a^h$) is active in s, if $a^h$ $\in {\bf A}^H$ and $g_1^h, \dots, g_m^h \; \in s$ and all inhibition rules and contravening rules of action a are passive in s, otherwise, the facilitating rule is passive in s. The set $A_{FAC}(s)$ is the set of actions for which there exists at least one active facilitating rule in s. The set $\overline{A}_{FAC}(s)$ is the set of actions for which there exists at least one facilitating rule and all facilitating rules are passive in s.
   	 
    	\item
    	An contravening rule ($g_1^h, \dots, g_m^h \;$ \textbf{contravenes} $a^h$) is active in s, if $a^h$ $\in {\bf A}^H$ and $g_1^h, \dots, g_m^h \; \in s$ and all inhibition rules and facilitating rules of action a are passive in s, otherwise, the contravening rule is passive in s. The set $A_{INT}(s)$ is the set of actions for which there exists at least one active contravening rule in s. 
    
    	\item
    	A dynamic causal law (a causes $f_1,\dots,f_n \;$ if $g_1,\dots,g_m \;$) is applicable in s, if $g_1,\dots,g_m \in s$.
    
    	\item
    	A static causal law ($f_1,\dots,f_n \;$ if $g_1,\dots,g_m \;$) is applicable in s, if $g_1,\dots,g_m \in s \;$.    

        \item
    	A dynamic causal law (a \textbf{influences} $f_1^h,\dots,f_n^h \;$ if $g_1,\dots,g_m \;$) is applicable in s, if $g_1,\dots,g_m \in s$ , and
    	$f_1^h,\dots,f_n^h \in F^H$, and
    	$f_i \in F (1 \leq i \leq n)$.   	 
    
    	\item
    	A static causal law ($g_1,\dots,g_m \;$ \textbf{influences} $f_1^h,\dots,f_n^h$) is applicable in s, if $g_1,\dots,g_m \in s \;$ , and
    	$f_1^h,\dots,f_n^h \in F^H$, and
    	$f_i \in F (1 \leq i \leq n)$.	 

        \item A forbidding rule \((g_1^h, \dots, g_m^h \; \mathbf{forbids~to~cause} \; f_1^h,\dots,f_n^h)\) is active in \( s \) if \( \{g_1^h, \dots, g_m^h\} \subseteq s \), where \( f_1^h, \dots, f_n^h, g_1^h, \dots, g_m^h \in {\bf F}^H \). The set \( F(s) \) is the set of fluents that are forbidden in $s+1$ when at least one active forbidding rule exists in \( s \). 

	\end{enumerate}
\end{definition}

Intuitively, the rules describe how actions may or may not occur, and how fluents are connected across and within states. Rule (1) states that certain fluents can inhibit an action, making it inactive in the current state. Rule (2) captures when fluents trigger an action, meaning it must occur if its triggers are satisfied and it is not inhibited. Rule (3) specifies when fluents allow an action, meaning the action is permitted but not enforced. Rules (4) and (5) extend this to human actions: some mental fluents can facilitate a human action, while others can contravene it. Rule (6) introduces dynamic causal laws, describing how actions under certain conditions cause fluents to hold in the next state. Rule (7) introduces static causal laws, which specify that if some fluents hold in a state, then others must co-hold in the same state. Rule (8) adds dynamic influence for mental fluents: when an action occurs under certain conditions, it brings about specified mental fluents in the next state. Rule (9) specifies static influence for mental fluents: when its conditions hold, the corresponding mental fluents must also hold in the same state. Finally, Rule (10) introduces forbids to cause, which regulates mental change by specifying that if certain mental fluents hold in the current state, then some mental fluents are not allowed to appear in the next state.

The output of the action language is in terms of trajectories. A mental-state trajectory consists of a sequence of valid transitions, represented as \\ $\langle s_0, A_1, s_1, A_2, \dots, A_n, s_n \rangle$, with alternating sets of mind-altering actions $A \subseteq \mathbf{A}$ and mental-states $s \in S$, following the constraints of the BG.

\begin{definition}[Trajectory]
\label{def:trajectory}
Let $D^{MT}(\bf A, F)$ be a domain description. A trajectory $\langle s_0,A_1,s_1,A_2,$ $\dots,$ $A_n,s_n\rangle$ of $D^{MT}(\bf A, F)$ is a sequence of sets of actions $A_i \subseteq A$ and states $s_i$ of $D^{MT}(\bf A, F)$ satisfying the following conditions for 0 $\leq$ i $<$ n:
\begin{enumerate}

\item $(s_i, A, s_{i+1})\in S \times 2^A \backslash \{\} \times S$

\item $A_T(s_i) \subseteq A_{i+1}$

\item $A_{FAC}(s_i) \subseteq A_{i+1}$


\item $\overline{\rm A}_T(s_i) \cap A_{i+1} = \emptyset$

\item $\overline{\rm A}_A(s_i) \cap A_{i+1} = \emptyset$

\item $A_I(s_i) \cap A_{i+1} = \emptyset$

\item $\overline{\rm A}_{FAC}(s_i) \cap A_{i+1} = \emptyset$

\item $A_{INT}(s_i) \cap A_{i+1} = \emptyset$

\item $|A_i \cap B| \leq 1 ~for~all~( noconcurrency~B )~\in~D^{MT}(\bf A, F).$

\item \( F(s_i) \cap s_{i+1} = \emptyset \) ~~(no forbidden fluents in $s_{i+1}$)

\end{enumerate}

\end{definition}

The set \( B \) in condition (9) represents a subset of actions restricted by noconcurrency constraints in \( D^{MT}({\bf A}, {\bf F}) \), ensuring that actions in \( B \) cannot execute simultaneously. This prevents conflicts where multiple actions attempt to modify the same fluent at the same time. While actions affecting distinct fluents can occur concurrently, those altering the same fluent are mutually exclusive, maintaining consistency in fluent updates and ensuring well-defined state transitions.

\begin{definition}[Action Observation Language]
\label{def:action observation language}
The action observation language of ${\cal C}_{MT}$ (similar to ${\cal C}_{TAID}$) consists of expressions of the following form:

\begin{tabular}{ll}
$(f \;\; \mathbf{ at } \;\;t ) \; \;$  $(a\;\;  \mathbf{ occurs\_at  }\;\; t )$ & $(8)$ \\

\end{tabular}

\noindent where $f \in {\bf F}$, $a$ is an action and $t \in \mathbb{N}$ is a point in time.

\end{definition}

The action observation language allows us to specify observations concerning the current state of mental fluents and the execution of actions that influence mental states. By combining observations with the causal laws of the domain description, we can generate plans, explanations, and predictions regarding the behavior of the system. This integration of observations and the domain description is referred to as an action theory.

\begin{definition}[Action Theory]
\label{def:action theory}
Let $D$ be a domain description and $O$ be a set of observations. The pair $(D,O)$ is called an action theory.
\end{definition}

The action theory forms the basis for constructing trajectory models, trajectories where all observations are satisfied, providing a structured representation of the system's dynamics over time. Trajectory models enable us to analyze and reason about the evolution of mental states and actions, allowing for a deeper understanding how the mental state domain operates and how it responds to different observations and changes.

\begin{definition}[Trajectory Model]
\label{def:trajectory model}
Let $(D,O)$ be an action theory. A trajectory $\langle s_0,A_1,s_1,A_2,$ $\dots,$ $A_n,s_n\rangle$ of $D$ is a trajectory model of $(D,O)$, if it satisfies all observations of $O$ in the following way:

\begin{enumerate}
  \item if $(f \;at\; t) \in O$, then $f \in s_t$
  \item if $(a \;occurs\_at\; t) \in O$, then $a \in A_{t+1}$.
\end{enumerate}

\end{definition}

We can observe that actions in a trajectory model can be actions executed by a rational software agent to influence mental fluents, or actions estimated to be executed by a human agent.

An important consequence of the $\mathbf{forbids~to~cause}$ rule is that it can render an action theory inconsistent. 
Since such a rule blocks any transition to a state containing its forbidden fluents, a trajectory that reaches such a state is invalid. 
If all candidate trajectories are thus blocked, then no trajectory model exists and the action theory is inconsistent. 
This regulative effect is intentional, ensuring that unwanted mental fluents cannot appear in valid states.

\begin{definition}[Action Theory Consistency]
\label{def:consistency}
An action theory $(D,O)$ is \emph{consistent} iff there exists a trajectory model of $(D,O)$. Otherwise, $(D,O)$ is \emph{inconsistent}.
\end{definition}

%
The central computational task is to determine, for a given ${\cal C}_{MT}$ action theory $(D,O)$, whether there exists a trajectory model of $(D,O)$. This baseline feasibility check underpins subsequent reasoning tasks such as planning, explanation, and query answering. 

\begin{definition}[${\cal C}_{MT}$ Decision Problem]
Given a ${\cal C}_{MT}$ theory $(D,O)$, decide if there exists a trajectory model of $(D,O)$.
\end{definition}

We need a mechanism to write queries about state dynamics. The Action Query Language provides a means to inquire about specific sequences of actions and their impact on fluents and states. This is achieved by specifying subsets of the action set and their corresponding occurrences in time.

\begin{definition}[Action Query Language] \label{def:QueryLanguage}
\label{def:action query language}
The action query language of ${\cal C}_{MT}$ regards assertions about executing sequences of actions with expressions that constitute trajectories. A query is of the following form:
$(f_1,\dots,f_n \; \mathbf{after}$  $A_i$ $\mathbf{ occurs\_at}$ $t_i, \dots, A_m$ $\mathbf{ occurs\_at}$ $t_m)$
\noindent where $f_1$, \dots, $f_n$ are fluent literals $\in {\bf F}$, $A_i$, \dots, $A_m$ are subsets of ${\bf A}$, and $t_i$, \dots, $t_m$ are points in time.

\end{definition}

By formulating queries in the action language, causal relationships between actions and their effects can be investigated, contributing to a deeper understanding of system dynamics, informed decision-making and controlled methods for automated planning.

\begin{definition}[Query Truth in Action Theory]
Let $(D,O)$ be an action theory and $Q$ a query.  
\begin{itemize}
  \item $Q$ holds \emph{skeptically} in $(D,O)$ iff $\tau \models Q$ for every trajectory model $\tau$ of $(D,O)$.
  \item $Q$ holds \emph{credulously} in $(D,O)$ iff $\tau \models Q$ for some trajectory model $\tau$ of $(D,O)$.
\end{itemize}
\end{definition}

Given a ${\cal C}_{MT}$ action theory $(D,O)$ and a query $Q$, the decision problem regards whether $Q$ holds in all (or some) trajectory models of $(D,O)$.


\subsection{Action Language in Answer Set Semantics}
\label{section:AnswerSetSemantics-ae}



In this section, we provide operational semantics for ${\cal C}_{MT}$ by translation into answer set programs. 
The goal is to construct an encoding such that the encoding has exactly one answer set for every trajectory model of a ${\cal C}_{MT}$ theory. 
Intuitively, each answer set represents a trajectory model: for a trajectory $\langle s_0,A_1,s_1,\dots,A_n,s_n\rangle$, the atoms \texttt{holds}$(f,t)$ indicate that fluent $f$ holds in state $s_t$, and the atoms \texttt{holds(occurs}(a),t\texttt{)} indicate that action $a$ occurs between $s_t$ and $s_{t+1}$. 
Conversely, each answer set uniquely determines a trajectory by the fluents and actions it contains, thereby establishing a one-to-one correspondence between answer sets and trajectory models and reducing the central reasoning problems of ${\cal C}_{MT}$ to ASP computation.

To make the presentation self-contained, we include translations from \cite{dworschak2008mbn} for the ${\cal C}_{TAID}$ language (expressions 1–7 in Definition~\ref{def:domain description language}), and then extend them with the new constructs of ${\cal C}_{MT}$. 
We organize the translations into distinct subsections, addressing the Action Description Language, the Action Observation Language, and the Action Query Language.

\subsubsection{Encoding of the Action Description Language}
\label{translation}

We consider the encoding of ${\cal C}_{TAID}$ action description language, which we extend with components of ${\cal C}_{MT}$.\\

\noindent Define symbols for a fluent f $\in \bf F^E$ and an action a $\in \bf A^E$ (extending \cite{dworschak2008mbn}).
\begin{verbatim}
fluent_e(f), action_e(a).
fluent(f) :- fluent_e(f).
action(a) :- action_e(a).
\end{verbatim}

\noindent Define symbols for a fluent e $\in \bf F^H$ and an action u $\in \bf A^H$ (extending \cite{dworschak2008mbn}).
\begin{verbatim}
mental_fluent(e), human_action(u).
fluent(e) :- mental_fluent(e).
action(u) :- human_action(u).
\end{verbatim}

\noindent Define a range of time points $0 \leq t \leq t\_max$, $t \in \mathbb{N}$ (as in \cite{dworschak2008mbn}).
\begin{verbatim}
time(0..t_max). #const t_max = N. 
\end{verbatim}






%

\noindent \textbf{Contradiction Constraint (as in \cite{dworschak2008mbn})}: A fluent $f \in {\bf F}$ and its negation $\neg f$ cannot hold simultaneously at the same time step $T \in \mathbb{N}$. In ASP, this is enforced as a constraint, ensuring that $\texttt{holds}(f,T)$ and $\texttt{holds(neg(f),T)}$ cannot both be true. The predicates $\texttt{fluent}(f)$ and $\texttt{time}(T)$ declare fluents and time steps.

\begin{verbatim}
:- holds(f,T), holds(neg(f),T), fluent(f), time(T).
\end{verbatim}

\noindent \textbf{Inertial Fluents (as in \cite{dworschak2008mbn})}: An inertial fluent persists across time steps unless modified by an action or a static causal law. In ASP, this is encoded by ensuring that if $\texttt{holds}(f,T)$ is true and not overridden, then $\texttt{holds}(f,T+1)$ holds by default. The predicates $\texttt{fluent}(f)$ and $\texttt{time}(T)$ track fluents over time.

\begin{verbatim}
holds(f,T+1) :- holds(f,T), not holds(neg(f,T+1)), not default(f),
fluent(f), time(T), time(T+1).
\end{verbatim}

\noindent \textbf{Non-Inertial Fluents (as in \cite{dworschak2008mbn})}: A non-inertial fluent resets to a default value unless explicitly updated. In ASP, this is encoded by inferring $\texttt{holds}(f,T)$ whenever $\texttt{default}(f)$ holds and $\texttt{holds(neg(f),T)}$ is not inferred. The predicates $\texttt{fluent}(f)$, $\texttt{default}(f)$, and $\texttt{time}(T)$ define non-inertial fluents and their default behavior.

\begin{verbatim}
holds(f,T) :- not holds(neg(f),T), default(f), fluent(f), time(T).
\end{verbatim}

\noindent \textbf{Dynamic Causal Law (as in \cite{dworschak2008mbn})}: A dynamic causal law $(a \; \mathbf{causes}$ $f$ $\mathbf{if}$ $g_1, \dots, g_n)$ states that if an action $a \in {\bf A}$ occurs at time $T \in \mathbb{N}$ and fluents $g_1, \dots, g_n \in {\bf F}$ hold, then a fluent $f \in {\bf F}$ must hold at time $T+1$. In ASP, one rule is generated for each dynamic causal law, ensuring that if $\texttt{holds(occurs}(a,T)\texttt{)}$ and $\texttt{holds}(g_1,T), \dots, \texttt{holds}(g_n,T)$ hold, then $\texttt{holds}(f,T+1)$ is inferred. The predicates $\texttt{fluent}(f)$ and $\texttt{fluent}(g_1), \dots, \texttt{fluent}(g_n)$ declare fluents, $\texttt{action}(a)$ declares the action, and $\texttt{time}(T)$ ensures valid time progression.

\begin{verbatim}
holds(f,T+1) :- holds(occurs(a),T), 
holds(g_1,T), ..., holds(g_n,T),
fluent(g_1), ..., fluent(g_n), 
fluent(f), action(a), time(T), time(T+1).
\end{verbatim}

\noindent \textbf{Static Causal Law (as in \cite{dworschak2008mbn})}: A static causal law $(f \; \mathbf{if} \; g_1, \dots, g_n)$ states that a fluent $f \in {\bf F}$ holds at time $T \in \mathbb{N}$ if fluents $g_1, \dots, g_n \in {\bf F}$ hold at the same time. In ASP, one rule is generated for each static causal law, ensuring that if $\texttt{holds}(g_1,T), \dots, \texttt{holds}(g_n,T)$ hold, then $\texttt{holds}(f,T)$ is inferred. The predicates $\texttt{fluent}(f)$ and $\texttt{fluent}(g_1), \dots, \texttt{fluent}(g_n)$ declare fluents, and $\texttt{time}(T)$ ensures valid time progression.

\begin{verbatim}
holds(f,T) :- holds(g_1,T), ..., holds(g_n,T), 
fluent(g_1), ..., fluent(g_n), fluent(f), time(T).
\end{verbatim}

\noindent \textbf{Inhibition Rule (as in \cite{dworschak2008mbn})}: An inhibition rule $(f_1, \dots, f_n \textbf{ inhibits } a)$ states that an action $a \in {\bf A}$ is prevented from occurring at time $T \in \mathbb{N}$ if fluents $f_1, \dots, f_n \in {\bf F}$ hold. In ASP, one rule is generated for each inhibition rule, ensuring that if $\texttt{holds}(f_1,T), \dots, \texttt{holds}(f_n,T)$ hold, then $\texttt{holds(ab(occurs}(a),T)\texttt{)}$ is inferred, marking the action as inhibited. The predicates $\texttt{fluent}(f_1),$ $\dots,$ $\texttt{fluent}(f_n)$ declare fluents, $\texttt{action}(a)$ declares the action, and $\texttt{time}(T)$ ensures valid time progression.

\begin{verbatim}
holds(ab(occurs(a)),T) :- holds(f_1,T), ..., holds(f_n,T), 
action(a), fluent(f_1), ..., fluent(f_n), time(T).
\end{verbatim}


As in ${\cal C}_{TAID}$, unlike standard ASP planning encodings, we do not use a choice rule for action generation, such as 

\texttt{\{ holds(occurs(A), T): action(A) \} = 1 :- T = 1..t\_max.} Instead, actions are determined by logical conditions: they occur if they are explicitly triggered, allowed, and not inhibited. This ensures that only valid and necessary actions are selected, avoiding unnecessary non-determinism.\\



\noindent \textbf{Triggering Rule (as in \cite{dworschak2008mbn})}: A triggering rule $(f_1, \dots, f_n \textbf{ triggers } a)$ states that an action $a \in {\bf A}$ occurs at time $T \in \mathbb{N}$ if fluents $f_1, \dots, f_n \in {\bf F}$ hold and no inhibition rule is active. In ASP, one rule is generated for each triggering rule, ensuring that if $\texttt{holds}(f_1,T), \dots, \texttt{holds}(f_n,T)$ hold and $\texttt{not holds(ab(occurs}(a,T)\texttt{))}$, then $\texttt{holds(occurs}(a,T)\texttt{)}$ is inferred. The predicates $\texttt{fluent}(f_1), \dots, \texttt{fluent}(f_n)$ declare fluents, $\texttt{action}(a)$ declares the action, and $\texttt{time}(T)$ ensures valid time progression.

\begin{verbatim}
holds(occurs(a),T) :- not holds(ab(occurs(a)),T), 
holds(f_1,T), ..., holds(f_n,T), 
fluent(f_1), ..., fluent(f_n), action(a), time(T).
\end{verbatim}

\noindent \textbf{Allowance Rule (as in \cite{dworschak2008mbn})}: An allowance rule $(f_1, \dots, f_n \textbf{ allows } a)$ states that an action $a \in {\bf A}$ is permitted to occur at time $T \in \mathbb{N}$ if fluents $f_1, \dots, f_n \in {\bf F}$ hold and no inhibition rule is active. In ASP, one rule is generated for each allowance rule, ensuring that if $\texttt{holds}(f_1,T), \dots, \texttt{holds}(f_n,T)$ hold and $\texttt{not holds(ab(occurs}(a,T)\texttt{))}$, then $\texttt{holds(allow(occurs}(a),T)\texttt{)}$ is inferred. The predicates $\texttt{fluent}(f_1), \dots, \texttt{fluent}(f_n)$ declare fluents, $\texttt{action}(a)$ declares the action, and $\texttt{time}(T)$ ensures valid time progression.

\begin{verbatim}
holds(allow(occurs(a)),T) :- not holds(ab(occurs(a)),T), 
holds(f_1,T), ..., holds(f_n,T), fluent(f_1), ..., fluent(f_n), 
action(a), time(T).
\end{verbatim}

\noindent \textbf{Ensure Exogenous Actions Can Always Occur (as in \cite{dworschak2008mbn})}: This rule ensures that an action $a \in {\bf A}$ is always permitted to occur at any time step $T \in \mathbb{N}$, regardless of specific preconditions. In ASP, this is encoded by inferring $\texttt{holds(allow(occurs}(a),T)\texttt{)}$ for every declared action $\texttt{action}(a)$ and time point $\texttt{time}(T)$. This guarantees that exogenous actions remain available throughout the execution.

\begin{verbatim}
holds(allow(occurs(a)),T) :- action(a), time(T).
\end{verbatim}

\noindent \textbf{No-Concurrency Constraint (as in \cite{dworschak2008mbn})}: The actions $a_1, \dots, a_n \in {\bf A}$ cannot occur simultaneously at time $T \in \mathbb{N}$. In ASP, this is encoded as a constraint that ensures selecting two or more actions at $T$ leads to inconsistency. The rule applies over declared actions $\texttt{action}(a_1), \dots, \texttt{action}(a_n)$ and time points $\texttt{time}(T)$.

\begin{verbatim}
:- time(T), 2 {holds(occurs(a_1),T) : 
action(a_1), ..., holds(occurs(a_n),T) : action(a_n)}.
\end{verbatim}

\subsubsection{Encoding of the Action Observation Language}

We consider the encoding of ${\cal C}_{TAID}$ action observation language, which we extend with components of ${\cal C}_{MT}$. The action observation language allows for the specification of observations about fluents and actions over time. Observations provide constraints on the initial state and subsequent state transitions, ensuring that execution traces align with known facts.\\

\noindent \textbf{Initial State Fluent Observations (as in \cite{dworschak2008mbn})}: A fluent $f \in {\bf F}$ that holds in the initial state at time $T = 0$ is directly asserted. In ASP, this is represented by $\texttt{holds}(f,0)$, ensuring that observed fluents are set at the start of execution, forming the basis for reasoning about subsequent state transitions.

\begin{verbatim}
holds(f,0).
\end{verbatim}



\noindent \textbf{Fluent Observations for Other States (as in \cite{dworschak2008mbn})}: If a fluent $f \in {\bf F}$ is observed to hold at time $T \in \mathbb{N}$, it must be enforced. In ASP, this is encoded as a constraint ensuring that $\texttt{holds}(f,T)$ must be true whenever $f$ is observed. The rule applies over declared fluents $\texttt{fluent}(f)$ and time points $\texttt{time}(T)$.

\begin{verbatim}
:- not holds(f,T), fluent(f), time(T).
\end{verbatim}

\noindent \textbf{Generate Possible Completions of the Initial State (as in \cite{dworschak2008mbn})}: A fluent $f \in {\bf F}$ in the initial state at $T = 0$ must either hold or its negation $\neg f$ must hold, but not both. In ASP, this is encoded using a default completion principle, ensuring that $\texttt{holds}(f,0)$ is inferred unless $\texttt{holds(neg(f),0)}$ is explicitly stated, and vice versa. The rule applies over declared fluents $\texttt{fluent}(f)$ at time step $\texttt{time}(0)$.

\begin{verbatim}
holds(f,0) :- not holds(neg(f),0).
holds(neg(f),0) :- not holds(f,0).
\end{verbatim}

\noindent \textbf{Exogenous Action Observations (as in \cite{dworschak2008mbn})}: An exogenous action $a \in {\bf A}$ that occurs at time $T \in \mathbb{N}$ is recorded as a fact. This corresponds to an external execution of $a$, independent of triggering, allowance, or inhibition conditions. In ASP, this is represented by asserting $\texttt{holds(occurs}(a),T\texttt{)}$, ensuring that the action is registered as having taken place. The rule applies over declared actions $\texttt{action}(a)$ and time points $\texttt{time}(T)$.

\begin{verbatim}
holds(occurs(a),T).  
\end{verbatim}

\noindent \textbf{Observed Non-Occurrence of Actions (as in \cite{dworschak2008mbn})}: If an action $a \in {\bf A}$ does not occur at time $T \in \mathbb{N}$, this must be explicitly recorded. This corresponds to cases where no triggering or allowance rule is active, or an inhibition rule prevents execution. In ASP, this is enforced as a constraint, ensuring that $\texttt{holds(neg(occurs}(a),T\texttt{))}$ is inferred when the action does not occur. The rule applies over declared actions $\texttt{action}(a)$ and time points $\texttt{time}(T)$.

\begin{verbatim}
:- not holds(neg(occurs(a)),T), action(a), time(T).  
\end{verbatim}

\noindent \textbf{Action Execution Conditions (as in \cite{dworschak2008mbn})}: An action $a \in {\bf A}$ occurs at time $T \in \mathbb{N}$ if at least one allowance rule $(f_1, \dots, f_n \mathbf{allows}~a)$ is active, no inhibition rule $(f_1, \dots, f_n \mathbf{inhibits}~a)$ applies, and no external constraint negates its occurrence. In ASP, this is encoded by ensuring that $\texttt{holds(occurs}(a),T\texttt{)}$ is inferred under these conditions. Additionally, if an action does not occur, its negation $\texttt{holds(neg(occurs}(a),T\texttt{))}$ is inferred. The rule applies over declared actions $\texttt{action}(a)$ and time points $\texttt{time}(T)$, where $T < t_{\max}$.

\begin{verbatim}
holds(occurs(a),T) :- holds(allow(occurs(a)),T), 
not holds(ab(occurs(a)),T), not holds(neg(occurs(a)),T), 
action(a), time(T), T < t_max.

holds(neg(occurs(a)),T) :- not holds(occurs(a),T), 
action(a), time(T), T < t_max.
\end{verbatim}

\subsubsection{Encoding of the Action Query Language}

We consider the encoding of ${\cal C}_{TAID}$ action query language, which we extend with components of ${\cal C}_{MT}$.\\



\noindent \textbf{Goal Achievement Constraint (as in \cite{dworschak2008mbn})}: The goal must be achieved in every valid plan. This corresponds to ensuring that at least one fluent configuration satisfying the goal holds. In ASP, this is enforced as a constraint ensuring that $\texttt{achieved}$ is inferred.

\begin{verbatim}
:- not achieved.
\end{verbatim}

\noindent \textbf{Initial Goal Satisfaction (as in \cite{dworschak2008mbn})}: The goal is considered achieved if it already holds at the initial time step $T=0$.

\begin{verbatim}
achieved :- achieved(0).
\end{verbatim}

\noindent \textbf{Persistence of Goal Achievement (as in \cite{dworschak2008mbn})}: If the goal is achieved at time $T+1$ but was not achieved at $T$, then it remains achieved for all subsequent time steps. The predicates $\texttt{time}(T)$ ensure valid time progression.

\begin{verbatim}
achieved :- achieved(T+1), not achieved(T), time(T), time(T+1).
\end{verbatim}

\noindent \textbf{Fluent-Based Goal Satisfaction (as in \cite{dworschak2008mbn})}: The goal is achieved at time $T$ if a set of fluents $f_1, \dots, f_n \in {\bf F}$ required for goal satisfaction hold. This corresponds to a static causal law of the form $(f_1, \dots, f_n \mathbf{if} \; \text{goal\_conditions})$. In ASP, this is encoded by inferring $\texttt{achieved}(T)$ when $\texttt{holds}(f_1,T), ..., \texttt{holds}(f_n,T)$ hold. 

\begin{verbatim}
achieved(T) :- holds(f_1, T), ..., holds(f_n, T), 
achieved(T+1), fluent(f_1), ..., fluent(f_n), time(T), time(T+1).
\end{verbatim}

\noindent \textbf{Final Goal Satisfaction (as in \cite{dworschak2008mbn})}: The goal must hold at the maximum time step $t_{\max}$. This ensures that the trajectory satisfies the required final state.

\begin{verbatim}
achieved(t_max) :- holds(f_1, t_max), ..., holds(f_n, t_max),
fluent(f_1), ..., fluent(f_n).
\end{verbatim}



\noindent \textbf{Execution of Allowed Actions (as in \cite{dworschak2008mbn})}: An action $a \in {\bf A}$ occurs at time $T \in \mathbb{N}$ if it is allowed by an active allowance rule $(f_1, \dots, f_n \mathbf{allows}~a)$, is not inhibited by an inhibition rule $(f_1, \dots, f_n \mathbf{inhibits} a)$, and the goal has not yet been achieved. In ASP, this ensures that $\texttt{holds(occurs}(a),T\texttt{)}$ is inferred only when these conditions hold. The predicates $\texttt{fluent}(f_1), ..., \texttt{fluent}(f_n)$ declare fluents, $\texttt{action}(a)$ declares actions, and $\texttt{time}(T)$ defines time steps.

\begin{verbatim}
holds(occurs(a), T) :- holds(allow(occurs(a)), T), not achieved(T), 
not holds(ab(occurs(a)), T), not holds(neg(occurs(a)), T), 
action(a), time(T).
\end{verbatim}

\noindent \textbf{Explicit Non-Occurrence of Actions (as in \cite{dworschak2008mbn})}: If an action $a \in {\bf A}$ does not occur at time $T \in \mathbb{N}$, this must be explicitly recorded. This corresponds to enforcing that an action not chosen in the trajectory is negated.

\begin{verbatim}
holds(neg(occurs(a)), T) :- not holds(occurs(a), T), 
action(a), time(T).
\end{verbatim}

\subsubsection{Encoding of the Mental State Language Extension}

We now proceed by presenting the translations for static and dynamic causal laws for mental state specifications introduced in ${\cal C}_{MT}$. \\

\noindent \textbf{Dynamic causal law for Mental Fluents (introduced in ${\cal C}_{MT}$)}: A dynamic causal law $(a \; \mathbf{influences} \; f_1^h, \dots, f_n^h \; \mathbf{if} \; g_1, \dots, g_m)$ specifies that an action $a \in {\bf A}$ causes mental fluents $f_1^h, \dots, f_n^h \in {\bf F}^H$ to hold at $T+1$ if fluents $g_1, \dots, g_m \in {\bf F}$ hold at $T \in \mathbb{N}$. The ASP encoding generates a rule for each mental fluent $f_i^h \in {\bf F}^H$ ($1 \leq i \leq n$), of the form: 

\begin{verbatim}
holds(f_i, T+1) :- holds(occurs(a), T), 
    holds(g_1, T), ..., holds(g_m, T),
    fluent(g_1), ..., fluent(g_m), mental_fluent(f_i),
    action(a), time(T).
\end{verbatim}



\noindent \textbf{Static causal law for Mental Fluents (introduced in ${\cal C}_{MT}$)}: A static causal law $(g_1, \dots, g_m \; \mathbf{influences} \; f_1^h,\dots,f_n^h)$ states that if fluents $g_1, \dots, g_m \in {\bf F}$ hold at time $T$, then mental fluents $f_1^h,\dots,f_n^h \in {\bf F}^H$ also hold at $T \in \mathbb{N}$.  The ASP encoding generates a rule for each mental fluent $f_i^h \in {\bf F}^H$ ($1 \leq i \leq n$), of the form: 

\begin{verbatim}
holds(f_i, T) :- holds(g_1, T), ..., holds(g_m, T),
        fluent(g_1), ..., fluent(g_m), mental_fluent(f_i),
        time(T).
\end{verbatim}

The following translations regard contravening and facilitating rules (introduced in ${\cal C}_{MT}$). Unlike the `inhibition' and `allowance' rules outlined in ${\cal C}_{TAID}$ \cite{dworschak2008mbn}, which impact actions in the current time step, `contravenes' and `facilitates' relate particularly about a human action $a^h \in {\bf A}^H$ and mental fluents $f^h \in {\bf F}^H$. As we have previously discussed; appraisal theories have suggested that emotional states have distinctive ``action tendencies'' \cite{roseman1994phenomenology}. Capturing these semantics, while seemingly redundant due to the general rules in the action language, supports knowledge representation.\\

\noindent \textbf{Facilitation Rule (introduced in ${\cal C}_{MT}$)}: A facilitation rule $(g_1^h,\dots,g_m^h$ $\mathbf{facilitates} \; a^h)$ states that if mental fluents $g_1^h, \dots, g_m^h \in {\bf F}^H$ hold at time $T \in \mathbb{N}$, and the action $a^h \in {\bf A}^H$ is not inhibited, then $a^h$ occurs. The ASP encoding ensures that $\texttt{holds(occurs}(a^h),T\texttt{)}$ is inferred if $\texttt{not holds(ab(occurs}(a^h),T)\texttt{)}$ holds.

\begin{verbatim}
holds(occurs(a), T) :- not holds(ab(occurs(a)), T),
        holds(g_1, T), ..., holds(g_m, T),
        mental_fluent(g_1), ..., mental_fluent(g_m),
        human_action(a), time(T).
\end{verbatim}

\noindent \textbf{Contravening Rule (introduced in ${\cal C}_{MT}$)}: A contravening rule $(g_1^h,\dots,g_m^h$ $\mathbf{contravenes} \; a^h)$ states that if mental fluents $g_1^h, \dots, g_m^h \in {\bf F}^H$ hold at time $T \in \mathbb{N}$, then the human action $a^h \in {\bf A}^H$ is inhibited from occurring. The ASP encoding infers $\texttt{holds(ab(occurs}(a^h),T\texttt{))}$ under these conditions.

\begin{verbatim}
holds(ab(occurs(a)), T) :- holds(g_1, T), ..., holds(g_m, T),
        mental_fluent(g_1), ..., mental_fluent(g_m),
        human_action(a), time(T).
\end{verbatim}

\noindent \textbf{Forbidding Rule (introduced in ${\cal C}_{MT}$)}: A forbidding rule $(g_1^h,\dots,g_m^h$ \\ $\mathbf{forbids~to~cause}$ $f_1^h, \dots, f_n^h)$ enforces constraints on mental state transitions in the belief graph. If mental fluents $g_1^h, \dots, g_m^h \in {\bf F}^H$ hold at time $T \in \mathbb{N}$, then fluents $f_1^h, \dots, f_n^h \in {\bf F}^H$ are forbidden to hold at $T+1$. This ensures that the belief graph’s constraints on state transitions are respected. The ASP encoding introduces an integrity constraint to prohibit transitions that violate these restrictions.

\begin{verbatim}
:- holds(f_1, T+1), ..., holds(f_n, T+1), 
        holds(g_1, T), ..., holds(g_m, T),
        mental_fluent(f_1), ..., mental_fluent(f_n),
        mental_fluent(g_1), ..., mental_fluent(g_m),
        time(T).
\end{verbatim}

In the upcoming section, we will explore a case study that applies a specialization of ${\cal C}_{MT}$ for emotional reasoning. The case study models emotion states and dynamics based on well-established psychological theories. By examining this specialized application, we aim to demonstrate the effectiveness and versatility of the ${\cal C}_{MT}$ framework for modeling different kinds of mental state dynamics.

\section{Case Study: Emotional Reasoning}
\label{sec:Case Study Emotion}

\noindent The section presents a characterization of emotion dynamics in ${\cal C}_{MT}$ with specialized subsets of fluents and mental state constraints adapted for emotional reasoning. This is achieved by formalizing different emotion theories; the Appraisal theory of Emotion (AE) by Roseman (1996) \cite{roseman1996appraisal}, Hedonic Emotion Regulation (HER) \cite{zaki2020integrating}, and Utilitarian Emotion Regulation (UER) \cite{tamir2007business}, capturing links between human emotions and their underlying causes in the environment. To this end, subsets of mental state fluents, called emotion fluents, together with a set of action rules and sets of mental state constraints are specified, capturing principles from emotion theories. 
By following the psychological theory of AE, an emotion state-space is defined with 108 emotional configurations (states), through which a set of 16 basic human emotions, according to AE, can be represented and explained.
By following the theories of HER and UER, different constraints for emotional-change are defined. HER focuses on augmenting positive emotions while diminishing negative ones. Conversely, UER seeks to promote particular emotions, including potentially negative ones, that enhance specific attributes, such as motivation or control, serving a utilitarian purpose in the long run. Consequently, HER and UER adopt contrasting principles for effecting emotional change. By examining the trajectories produced through the application of either HER or UER within the state-space of AE, we can compare and evaluate their respective behaviors.


\subsection{Emotion Theories: AE, HER and UER}
\label{sec:Emotion Theories}

The \textbf{Appraisal theory of Emotion (AE)} by Roseman (1996) \cite{roseman1996appraisal} proposes that emotions are caused by an appraisal of a situation in terms of 1) being consistent or inconsistent with needs, importance of the situation, 2) being consistent or inconsistent with goals, the attainability/potential to achieve goals, 3) who/what is accountable/caused the situation, which can be the environment, others, or oneself, and 4) as being easy or difficult to control. According to AE, the difference between goal consistency/attainability and need consistency/importance determines negative, stable and positive emotions. More intense negative emotions (e.g., Anger or Fear) arise when the need consistency is greater than the goal consistency, while less intense negative emotions can arise when both the need consistency and goal consistency are low. On the other hand, positive emotions (e.g., Joy or Liking) arise when the goal consistency is greater than the need consistency, or when both are high. By ranking consistency values as \emph{Low $<$ Undecided $<$ High} and by looking at the difference between need and goal consistency, positive and negative emotions can be distinguished.

\textbf{Hedonic Emotion Regulation (HER)} \cite{zaki2020integrating} is a theory for regulating emotions, guided by the goals to increase positive emotion and decrease negative emotion. According to HER, both of these emotion regulation goals are associated with improved well-being, where decreasing of negative emotion has been most effective \cite{ortner2018roles}. The principles of HER can be applied to reason about emotional change. For instance, the relation between goal attainability and goal importance has been empirically explored, showing that goal attainability, rather than goal importance, was positively linked to well–being \cite{buhler2019closer}. Another empirical study analyzed self-responsibility and emotions, and showed that accountability does not by itself affect whether the emotion is negative or positive. Nevertheless, when something else than when an individual (self/other) is perceived as accountable, less negative emotions appear \cite{passyn2006self}. Another empirical study examined the relationship between control potential and emotions \cite{roseman1996appraisal}. The study found that the perception of high or low control potential influenced the experience of accommodating emotions or contending emotions, respectively. However, the study did not find a direct effect of control potential on the experience of positive or negative emotions.

By analyzing the four dimensions of emotion as proposed by AE (Need consistency, Goal consistency, Accountability, and Control potential), we can derive meaningful interpretations about the relationship between HER and AE in the context of hedonic emotional change. These interpretations are summarized in Table \ref{tab:HER-interpretations}.

\begin{table}[h]
\centering
\caption{Hedonic Emotion Regulation as basis for Principles of Change}
\label{tab:HER-interpretations}
{\tablefont\begin{tabular}{@{\extracolsep{\fill}} p{0.48\linewidth}  p{0.45\linewidth}}
	\hline
	\textbf{Hedonic Emotion Regulation} & \textbf{Principle of Change}\\
	\hline
	\textbf{Need Consistency/Importance:} A high need consistency is associated with positive emotions only when the goal consistency is high, but negative if the goal consistency is low (supported by \cite{buhler2019closer}). & High need consistency maximizes positive emotion only when goal consistency is high. \\
	\hline
	\textbf{Goal Consistency/Attainability:} A high goal consistency is associated with positive emotions. In particular when the need consistency is high (supported by \cite{buhler2019closer}). & High goal consistency maximizes positive emotion. \\
	\hline
	\textbf{Accountability:} Accountability does not affect whether the emotion is negative or positive. Nevertheless, when the environment is perceived as primarily accountable for the situation, less negative emotions appear than when an individual (self/other) is perceived as accountable (supported by \cite{passyn2006self}). & Accountability to environment maximizes positive emotions. \\
	\hline
	\textbf{Control Potential:} Control Potential does not affect whether the emotion is negative or positive. Nevertheless, a control potential above low is often associated with accommodating emotions, closer to positive emotions (supported by \cite{roseman1996appraisal}). & High control potential maximizes positive emotion when there is a non-negative balance between goal consistency and need consistency, i.e., when goal consistency is equal or higher than need consistency.\\
	\hline
\end{tabular}}
\end{table}

\textbf{Utilitarian Emotion Regulation (UER)} \cite{tamir2007business} is a theory emotion regulation that, in contrast to HER, is guided by the goals to experience emotions in the short-term for a long-term utilitarian purpose. 
This may involve accepting a temporary negative emotion to increase their capability of dealing with a situation. 
For instance, valuing long-term goals (i.e., focusing on need, importance and motivation) \cite{dweck2017needs} over immediate pleasure makes individuals more likely to engage in UER. Furthermore, a study on goal-setting \cite{schunk2001self} found that when individuals perceive a discrepancy between their present performance and a desired goal, it can lead to dissatisfaction. This dissatisfaction, in turn, can serve as a motivator for increased effort and striving towards the goal. Another study, exploring accountability and emotion \cite{autry1985locus}, suggests that self-accountability for a situation has utilitarian gains by having effects on the individual's ``locus of control''. Moreover, a study exploring control and emotion \cite{tamir2008hedonic} found that increasing control potential, such as through an increased level of anger, can have utilitarian gains. Another study on the emotion of fear \cite{tamir2009choosing} proposes the concept of ``fear as function for goal pursuit'', highlighting its role in motivating goal-directed behavior. Yet another study \cite{parrott2002functional} suggests that ``negative'' emotions can modify an individual's ``readiness to think and act, and their potential functional utility''.

By analyzing the four dimensions of emotion as proposed by AE (Need consistency, Goal consistency, Accountability, and Control potential), we can derive meaningful interpretations about the relationship between UER and AE in the context of utilitarian emotional change. These interpretations are summarized in Table \ref{tab:UER-interpretations}.

\begin{table}[h]
\centering
\caption{Utilitarian Emotion Regulation as basis for Principles of Change}
\label{tab:UER-interpretations}
{\tablefont\begin{tabular}{@{\extracolsep{\fill}} p{0.48\linewidth}  p{0.45\linewidth}}
	\hline
	\textbf{Utilitarian Emotion Regulation} & \textbf{Principle of Change}\\
	\hline
	\textbf{Need Consistency/Importance:} Valuing long-term goals over immediate pleasure makes individuals more likely to engage in UER. They are willing to endure temporary negative emotions to maximize utility and achieve important goals (supported by \cite{dweck2017needs}). & High need consistency maximizes utility. \\
	\hline
	\textbf{Goal Consistency/Attainability:} Perceiving important goals as difficult to achieve in the short-term can motivate individuals to engage in UER strategies to increase their chances of long-term success (supported by \cite{schunk2001self}). & Low or Undecided goal consistency maximizes utility if the need consistency is high. \\
	\hline
	\textbf{Accountability:} Individuals perceiving themselves as primarily accountable for their long-term goals drive them to regulate emotions in a utilitarian manner. Taking responsibility for emotional experiences enhances their ability to deal with challenges effectively. Furthermore, if the control potential is high, the environment can be perceived as accountable, where an individual may be able to affect aspects of the environment for utilitarian gains (supported by \cite{autry1985locus}). & Accountability to self maximizes utility. Accountability to environment maximizes utility if the control potential is high. \\
	\hline
	\textbf{Control Potential:} Sense of control is an important utilitarian characteristic. Accepting and regulating temporary negative emotions to increase perceived control contribute to long-term goal attainment (supported by \cite{tamir2008hedonic}). & High control potential maximizes utility.\\
	\hline
\end{tabular}}
\end{table}

Following the interpretations of HER and UER, we can define transition constraints in an emotion specialized BG. This will allow certain sets of trajectories depending on which emotional-change theory we model. In both cases, we use the state space defined by AE.

    


In the following section, we present a formalization of these theories in terms of ${\cal C}_{MT}$. We first define an emotion state in terms of the dimensions of AE. We then define a set of transition constraints that restrict particular configurations (states) of AE-appraisals to arise. Two sets of transition constraints are defined, constituting two different sub-graphs of valid transitions in the AE state-space, a so-called \emph{hedonic emotion state} and an \emph{utilitarian emotion state}, based on HER and UER principles, respectively. 

\subsection{Formalizing Emotional Reasoning}

Components of AE and HER (resp. UER) are formalized as a specialized Belief Graph 
to reason about emotion states and hedonic (resp. utilitarian) emotional change to reduce unintended emotional side-effects. The constraints of the specialized BG serve as safety restrictions for emotion-influencing actions, applicable for different types of emotional reasoning settings.  

Recall that AE defines emotions as a composition of an individual's appraisal of a situation, in terms of need consistency (importance), goal consistency (attainability), accountability (who/what) and control potential. By following this definition of emotional causes, we can define \emph{emotion fluent}, a changeable emotion variable.

\begin{definition}[Emotion fluent]
\label{def:emotion-fluent}

An emotion fluent is a mental fluent $f(c, v)$, a ground atom of arity 2, where $c \in C$, and $v \in V_c$, such that
$C = \{ne, go, ac, co\}$ is a set of constants denoting psychological classes and $V = \{V_{ne}, V_{go}, V_{ac}, V_{co}\}$ a set of total ordered sets of constants denoting psychological values for each class of $C$, where:

\indent $V_{ne} = \{high, low, undecided\}$\\
\indent $V_{go} = \{high, low, undecided\}$\\
\indent $V_{ac} = \{self, other, environment, undecided\}$\\
\indent $V_{co} = \{high, low, undecided\}$.\\

\noindent where $ne$ denotes need consistency, $go$ denotes goal consistency, $ac$ denotes accountability and $co$ denotes control potential.

\end{definition}

By defining a set of emotions following AE in this way, and by utilizing different principles of emotion regulation, we can specify preferred (e.g., hedonic or utilitarian) transitions between emotion states. In the following subsection, we specify a BG to reason about emotional transitions.


Following the AE-theory, 16 emotion states are specified, corresponding to each basic emotion: \{Anger, Dislike, Disgust, Sadness, Hope, Frustration, Fear, Distress, Joy, Liking, Pride, Surprise, Relief, Regret, Shame, Guilt\}. It is important to note that various other emotion states can be defined (in total 108) by different combinations of the four AE dimensions, and it is not necessary, nor always preferred, to assign a specific label, such as ``Joy'', to each of them. While labels aid in expressing these states in a human readable way, they do not inherently contribute to the functionality of a system's reasoning. It is worth noting that emotional expressions vary across individuals, cultures, languages, and other factors. Therefore, it is crucial for the system to analyze emotion states using a multidimensional format. We can model the states and transitions as an EG that represents valid emotional change given a recognized emotion state configuration. In Fig. \ref{fig:emotion graph}, examples of emotion states are presented.

 \begin{figure}[ht!]
 \centering
  \includegraphics[width=0.8\textwidth]{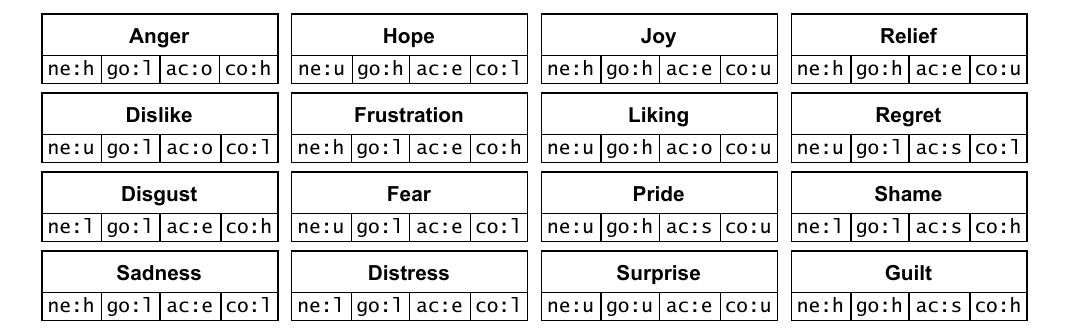}
  \caption{Emotion states following the Appraisal theory of Emotion by Roseman (1996) \cite{roseman1996appraisal}. Each emotion state is here expressed by an intuitive ``emotion'' label on top, and an appraisal configuration consisting of a set of variable:value pairs below. The variables are: ne = need consistency, go = goal consistency, ac = accountable, co = control potential. The values are: l = low, h = high, u = undecided, o = other, s = self, e = environment. }
  \label{fig:emotion graph}
 \end{figure}




The specialized belief graph for emotion is captured by a given program specified by the semantics of the action language ${\cal C}_{MT}$, with constraints in terms of HER or UER, serving as restrictions for hedonic contra utilitarian emotional change, presented in the following section.

\subsection{Action Language Emotion Specification}

The ${\cal C}_{MT}$ alphabet is specialized with
an emotion related vocabulary and causal laws to specify fluents and actions according to AE-theory. 

The emotional reasoning semantics is characterized by the constraints of a specialized BG for emotion, specified through a set of static causal laws and mental state constraints. In this way, we can restrict states and state-transitions to comply with principles of emotional change.

For any particular application, we need to define a BG that, based on application specific interaction goals and relevant theories, avoids unintended mental states. This specifies an BG with a subset of transitions (in the fully connected graph) that is considered valid. We define a hedonic theory specification that follows principles of HER, aiming to increase positive emotion and decrease negative emotion (see more details of HER in Section \ref{sec:Emotion Theories}). In an abstraction of HER, the hedonic emotion specification is defined to not allow entering a negative emotion state and to preserve a positive balance in emotion. 

Recall Figure~\ref{fig:Trajectory-space}, which illustrates the Belief Graph representing possible trajectories over environmental and mental states. Transitions (edges) in this graph are defined by a set of dynamic causal laws of the form $(a~\mathbf{causes}~f_1,\dots,f_n~\mathbf{if}~g_1, \dots, g_m)$, capturing how actions produce observable changes. Overlaid on these are forbidding rules of the form $(g_1^h,\dots,g_m^h~\mathbf{forbids~to~cause}~f_1^h, \dots, f_n^h)$, which specify mental state transitions that are disallowed based on psychological theory. The red arrows in the graph correspond to environmental transitions that are excluded because they would cause such forbidden mental transitions. In this way, the Belief Graph integrates both causal dynamics and psychological constraints to filter out inadmissible action trajectories.

The hedonic theory specification can thus be formally expressed as follows.

\begin{definition}[Hedonic emotion theory specification]
\label{def:hedonic_forbid}
Let \( D^{MT}_{AE}(\mathbf{A}, \mathbf{F}) \) be a domain description over the ${\cal C}_{MT}$ action language, with a vocabulary consisting of emotion fluents (as defined in Definition~\ref{def:emotion-fluent}) and actions over them.
A hedonic theory specification in \( D^{MT}_{AE}(\bf A, F) \) is a collection of causal laws as follows:

\begin{enumerate}
    \footnotesize
\item $f(c_{\text{ne}}, \text{high}) \; \mathbf{forbids~to~cause} \; f(c_{\text{go}}, \text{low}).$
\item $f(c_{\text{ne}}, \text{high}) \; \mathbf{forbids~to~cause} \; f(c_{\text{go}}, \text{undecided}).$
\item $f(c_{\text{ne}}, \text{high}) \; \mathbf{forbids~to~cause} \; f(c_{\text{go}}, \text{high}).$
\item $f(c_{\text{ne}}, \text{undecided}) \; \mathbf{forbids~to~cause} \; f(c_{\text{go}}, \text{low}).$
\item $f(c_{\text{ne}}, \text{undecided}) \; \mathbf{forbids~to~cause} \; f(c_{\text{go}}, \text{undecided}).$
\item $f(c_{\text{ne}}, \text{undecided}) \; \mathbf{forbids~to~cause} \; f(c_{\text{go}}, \text{high}).$

\item $f(c_{\text{go}}, \text{high}) \; \mathbf{forbids~to~cause} \; f(c_{\text{go}}, \text{low}).$
\item $f(c_{\text{go}}, \text{high}) \; \mathbf{forbids~to~cause} \; f(c_{\text{go}}, \text{undecided}).$
\item $f(c_{\text{go}}, \text{undecided}) \; \mathbf{forbids~to~cause} \; f(c_{\text{go}}, \text{low}).$
\item $f(c_{\text{go}}, \text{low}) \; \mathbf{forbids~to~cause} \; f(c_{\text{co}}, \text{high}).$

\item $\{ f(c_{\text{ne}}, \text{high}), f(c_{\text{go}}, \text{low}), f(c_{\text{ac}}, \text{other}) \} \; \mathbf{forbids~to~cause} \; f(c_{\text{co}}, \text{high}).$
\item $\{ f(c_{\text{ne}}, \text{high}), f(c_{\text{go}}, \text{undecided}), f(c_{\text{ac}}, \text{other}) \} \; \mathbf{forbids~to~cause} \; f(c_{\text{co}}, \text{high}).$
\item $\{ f(c_{\text{ne}}, \text{high}), f(c_{\text{go}}, \text{low}), f(c_{\text{ac}}, \text{self}) \} \; \mathbf{forbids~to~cause} \; f(c_{\text{co}}, \text{high}).$
\item $\{ f(c_{\text{ne}}, \text{high}), f(c_{\text{go}}, \text{undecided}), f(c_{\text{ac}}, \text{self}) \} \; \mathbf{forbids~to~cause} \; f(c_{\text{co}}, \text{high}).$
\item $\{ f(c_{\text{ne}}, \text{high}), f(c_{\text{go}}, \text{low}), f(c_{\text{ac}}, \text{undecided}) \} \; \mathbf{forbids~to~cause} \; f(c_{\text{co}}, \text{high}).$
\item $\{ f(c_{\text{ne}}, \text{high}), f(c_{\text{go}}, \text{undecided}), f(c_{\text{ac}}, \text{undecided}) \} \; \mathbf{forbids~to~cause} \; f(c_{\text{co}}, \text{high}).$

\end{enumerate}
\end{definition}


We further define an utilitarian emotion theory specification that follows principles of UER, aiming to prioritize particular emotional dimensions (e.g., need\_consistency) that are associated with increased utilitarian gain (see more details of UER in Section \ref{sec:Emotion Theories}).
In an abstraction of UER, the utilitarian emotion theory specification is defined to allow entering a
negative emotion state in the short-term if it may result in a long-term utilitarian gain.

\begin{definition}[Utilitarian emotion theory specification]
\label{def:utilitarian_forbid}
Let \( D^{MT}_{AE}(\mathbf{A}, \mathbf{F}) \) be a domain description over the ${\cal C}_{MT}$ action language, with a vocabulary consisting of emotion fluents (as defined in Definition~\ref{def:emotion-fluent}) and actions over them.
A utilitarian theory specification in \( D^{MT}_{AE}(\bf A, F) \) is a collection of causal laws as follows:

\begin{enumerate}
    \footnotesize
\item $f(c_{\text{ne}}, \text{low}) \; \mathbf{forbids~to~cause} \; f(c_{\text{ne}}, \text{undecided}).$
\item $f(c_{\text{ne}}, \text{low}) \; \mathbf{forbids~to~cause} \; f(c_{\text{ne}}, \text{high}).$
\item $f(c_{\text{ne}}, \text{undecided}) \; \mathbf{forbids~to~cause} \; f(c_{\text{ne}}, \text{low}).$
\item $f(c_{\text{ne}}, \text{undecided}) \; \mathbf{forbids~to~cause} \; f(c_{\text{ne}}, \text{high}).$

\item $f(c_{\text{go}}, \text{low}) \; \mathbf{forbids~to~cause} \; f(c_{\text{ne}}, \text{low}).$
\item $f(c_{\text{go}}, \text{low}) \; \mathbf{forbids~to~cause} \; f(c_{\text{ne}}, \text{undecided}).$
\item $f(c_{\text{go}}, \text{undecided}) \; \mathbf{forbids~to~cause} \; f(c_{\text{ne}}, \text{low}).$
\item $f(c_{\text{go}}, \text{undecided}) \; \mathbf{forbids~to~cause} \; f(c_{\text{ne}}, \text{undecided}).$
\item $f(c_{\text{go}}, \text{high}) \; \mathbf{forbids~to~cause} \; f(c_{\text{ne}}, \text{low}).$
\item $f(c_{\text{go}}, \text{high}) \; \mathbf{forbids~to~cause} \; f(c_{\text{ne}}, \text{undecided}).$
\item $f(c_{\text{go}}, \text{high}) \; \mathbf{forbids~to~cause} \; f(c_{\text{ne}}, \text{high}).$

\item $f(c_{\text{ac}}, \text{undecided}) \; \mathbf{forbids~to~cause} \; f(c_{\text{ac}}, \text{self}).$
\item $f(c_{\text{ac}}, \text{undecided}) \; \mathbf{forbids~to~cause} \; f(c_{\text{ac}}, \text{other}).$
\item $f(c_{\text{ac}}, \text{undecided}) \; \mathbf{forbids~to~cause} \; f(c_{\text{ac}}, \text{environment}).$
\item $f(c_{\text{ac}}, \text{other}) \; \mathbf{forbids~to~cause} \; f(c_{\text{ac}}, \text{undecided}).$
\item $f(c_{\text{ac}}, \text{other}) \; \mathbf{forbids~to~cause} \; f(c_{\text{ac}}, \text{self}).$
\item $f(c_{\text{ac}}, \text{other}) \; \mathbf{forbids~to~cause} \; f(c_{\text{ac}}, \text{environment}).$
\item $f(c_{\text{ac}}, \text{environment}) \; \mathbf{forbids~to~cause} \; f(c_{\text{co}}, \text{low}).$
\item $f(c_{\text{ac}}, \text{environment}) \; \mathbf{forbids~to~cause} \; f(c_{\text{co}}, \text{undecided}).$

\item $f(c_{\text{co}}, \text{low}) \; \mathbf{forbids~to~cause} \; f(c_{\text{co}}, \text{undecided}).$
\item $f(c_{\text{co}}, \text{low}) \; \mathbf{forbids~to~cause} \; f(c_{\text{co}}, \text{high}).$
\item $f(c_{\text{co}}, \text{undecided}) \; \mathbf{forbids~to~cause} \; f(c_{\text{co}}, \text{low}).$
\item $f(c_{\text{co}}, \text{undecided}) \; \mathbf{forbids~to~cause} \; f(c_{\text{co}}, \text{high}).$

\end{enumerate}
\end{definition}

These sets of $\mathbf{forbids~to~cause}$ rules are encoded as integrity constraints. Given a logic program $P_{EG}$ (based on the translation in Section \ref{translation}), we attach sets of integrity constraints representing HER (Listing \ref{lst:implementation-HER}) or UER (Listing \ref{lst:implementation-UER}) based principles. 
These integrity constraints are independent of the main logic program $P_{EG}$. By extending the program with different sets of integrity constraints, based on different emotion-regulation principles, we can analyze and compare their respective trajectories. 

\begin{lstlisting}[caption={Integrity Constraints for Hedonic Emotion Regulation.},captionpos=t,label={lst:implementation-HER}, language=]
% HEDONIC EMOTION REGULATION
% Need consistency can only increase if goal consistency is high
:- holds(mental_fluent(need_consistency, high), T+1), 
   holds(mental_fluent(goal_consistency, low), T), time(T).
:- holds(mental_fluent(need_consistency, high), T+1), 
   holds(mental_fluent(goal_consistency, undecided), T), time(T).
:- holds(mental_fluent(need_consistency, undecided), T+1), 
   holds(mental_fluent(goal_consistency, low), T), time(T).
:- holds(mental_fluent(need_consistency, undecided), T+1), 
   holds(mental_fluent(goal_consistency, undecided), T), time(T).

% Goal consistency cannot decrease
:- holds(mental_fluent(goal_consistency, low), T+1), 
   holds(mental_fluent(goal_consistency, V1), T), V1 != low, time(T).
:- holds(mental_fluent(goal_consistency, undecided), T+1), 
   holds(mental_fluent(goal_consistency, V1), T), V1 != undecided, time(T).

% Prevent high control potential if goal consistency is low
:- holds(mental_fluent(control_potential, high), T+1), 
   holds(mental_fluent(goal_consistency, low), T), time(T).

% Prevent control potential from increasing to high
% when accountability is not the environment
% and there is a negative balance between goal and need consistency
:- holds(mental_fluent(control_potential, high), T+1), 
   holds(mental_fluent(need_consistency, high), T), 
   holds(mental_fluent(goal_consistency, low), T), 
   holds(mental_fluent(accountability, other), T), time(T).

:- holds(mental_fluent(control_potential, high), T+1), 
   holds(mental_fluent(need_consistency, high), T), 
   holds(mental_fluent(goal_consistency, undecided), T), 
   holds(mental_fluent(accountability, other), T), time(T).

:- holds(mental_fluent(control_potential, high), T+1), 
   holds(mental_fluent(need_consistency, high), T), 
   holds(mental_fluent(goal_consistency, low), T), 
   holds(mental_fluent(accountability, self), T), time(T).

:- holds(mental_fluent(control_potential, high), T+1), 
   holds(mental_fluent(need_consistency, high), T), 
   holds(mental_fluent(goal_consistency, undecided), T), 
   holds(mental_fluent(accountability, self), T), time(T).

:- holds(mental_fluent(control_potential, high), T+1), 
   holds(mental_fluent(need_consistency, high), T), 
   holds(mental_fluent(goal_consistency, low), T), 
   holds(mental_fluent(accountability, undecided), T), time(T).

:- holds(mental_fluent(control_potential, high), T+1), 
   holds(mental_fluent(need_consistency, high), T), 
   holds(mental_fluent(goal_consistency, undecided), T), 
   holds(mental_fluent(accountability, undecided), T), time(T).
\end{lstlisting}

\begin{lstlisting}[caption={Integrity Constraints for Utilitarian Emotion Regulation.},captionpos=t,label={lst:implementation-UER}, language=]
% UTILITARIAN EMOTION REGULATION
% Need consistency can only be influenced to high.
:- holds(mental_fluent(need_consistency, low), T+1), 
    holds(mental_fluent(need_consistency, V1), T), V1 != low, time(T).
:- holds(mental_fluent(need_consistency, undecided), T+1), 
    holds(mental_fluent(need_consistency, V1), T), V1 != undecided, time(T).

% The goal consistency can only be influenced to low or undecided, 
% and the need consistency must be high.
:- holds(mental_fluent(goal_consistency, low), T+1), 
    holds(mental_fluent(need_consistency, low), T), time(T).
:- holds(mental_fluent(goal_consistency, low), T+1), 
    holds(mental_fluent(need_consistency, undecided), T), time(T).
:- holds(mental_fluent(goal_consistency, undecided), T+1), 
    holds(mental_fluent(need_consistency, low), T), time(T).
:- holds(mental_fluent(goal_consistency, undecided), T+1), 
    holds(mental_fluent(need_consistency, undecided), T), time(T).
:- holds(mental_fluent(goal_consistency, high), T+1), 
    holds(mental_fluent(need_consistency, undecided), T), time(T).
:- holds(mental_fluent(goal_consistency, high), T+1), 
    holds(mental_fluent(need_consistency, high), T), time(T).

% Accountability can only be influenced to self or environment. 
% If environment, then the control must be high.
:- holds(mental_fluent(accountability, undecided), T+1), 
    holds(mental_fluent(accountability, V1), T), V1 != undecided, time(T).
:- holds(mental_fluent(accountability, other), T+1), 
    holds(mental_fluent(accountability, V1), T), V1 != other, time(T).
:- holds(mental_fluent(accountability, environment), T+1), 
    holds(mental_fluent(control_potential, low), T), time(T).
:- holds(mental_fluent(accountability, environment), T+1), 
    holds(mental_fluent(control_potential, undecided), T), time(T).

% Control potential can only be influenced to High.
:- holds(mental_fluent(control_potential, low), T+1), 
    holds(mental_fluent(control_potential, V1), T), V1 != low, time(T).
:- holds(mental_fluent(control_potential, undecided), T+1), 
    holds(mental_fluent(control_potential, V1), T), V1 != undecided, time(T).
\end{lstlisting}

\section{Formal Analysis}
\label{sec:Formal Analysis}

We start by establishing a link between ${\cal C_{MT}}$ and answer set semantics, supporting implementations of the action language in Answer Set Programming (ASP). In the proceeding parts of the section, we show that \textbf{forbids to cause} rules imply relations between mental fluents, allowing us to analyze Safety properties based on these relations (considering the invariance principle \cite{hansen2003algorithms}) to be preserved in trajectories. 

To support the formal analysis, we define an operator that transforms an action theory (and query, if applicable) into a logic program $P_{MT}$ based on the high-level action language ${\cal C}_{MT}$. 

\begin{definition}[${\cal C}_{MT}$ Logic Program]
\label{def:logic_program_generation_operator}
Let $D^{MT}$ be a ${\cal C}_{MT}$ domain description and $O_{\text{initial}}$ a set of initial observations. We define the mapping function 
$P_{MT} := \Pi(D^{MT}, O_{\text{initial}})$ 
such that $P_{MT}$ is a ground ASP program where, for every expression $\varphi$ in $D^{MT} \cup O_{\text{initial}}$, $\Pi(\varphi)$ denotes the ASP translation of $\varphi$, and $P_{MT}$ includes all such translations. 
Let $\mathit{AS}(P_{MT})$ denote the set of answer sets of the logic program $P_{MT}$. 
\end{definition}

In the following subsections, we proceed with the formal analysis.

\subsection{Analysis of Translation to Answer Set Semantics}

In order to define the semantics of ${\cal C}_{MT}$, we characterize trajectory models in terms of answer sets. 
This is formalized by the following theorem:

\begin{theorem}
\label{theorem:C-MT-to-ASP-translation}
Let $(D^{MT},O_{\text{initial}})$ be an action theory such that $O_{\text{initial}}$ specifies fluent observations in the initial state. Let $Q$ be a query, and define:
\[
A_Q = \{(a~\text{occurs\_at}~t_k) \mid a \in A_{k+1},\; 0 \leq k < t_{\max} \}.
\]
Let ${\cal P}_{MT} := \Pi(D^{MT}, O_{\text{initial}} \cup A_Q)$ be the ${\cal C}_{MT}$ logic program obtained via the translation operator $\Pi$.

\begin{enumerate}
    \item If there is a trajectory model $\langle s_0, A_1, s_1, \dots, A_{t_{\max}}, s_{t_{\max}} \rangle$ of ${\cal C}_{MT}(D^{MT}, O_{\text{initial}} \cup A_Q)$, then there is an answer set ${\cal A}$ of ${\cal P}_{MT}$ such that for all time points $0 \leq k \leq t_{\max}$:
    \begin{enumerate}
        \item $holds(f,k) \in {\cal A}$ iff $f \in s_k$,
        \item $holds(neg(f),k) \in {\cal A}$ iff $f \notin s_k$,
        \item $holds(occurs(a),k) \in {\cal A}$ iff $a \in A_{k+1}$,
        \item $holds(neg(occurs(a)),k) \in {\cal A}$ iff $a \notin A_{k+1}$.
    \end{enumerate}
    
    \item If there is an answer set ${\cal A}$ of ${\cal P}_{MT}$ and for each $0 \leq k \leq t_{\max}$,
    \begin{enumerate}
        \item $s_k = \{f \mid holds(f,k) \in {\cal A}\} \cup \{\neg f \mid holds(neg(f),k) \in {\cal A}\}$,
        \item $A_{k+1} = \{a \mid holds(occurs(a),k) \in {\cal A}\}$,
    \end{enumerate}
    then there is a trajectory model $\langle s_0, A_1, s_1, \dots, A_{t_{\max}}, s_{t_{\max}} \rangle$ of ${\cal C}_{MT}(D^{MT}, O_{\text{initial}} \cup A_Q)$.
\end{enumerate}
\end{theorem}

\begin{proof}
We prove Theorem~\ref{theorem:C-MT-to-ASP-translation} in two directions using induction over the maximal time point $t_{\max}$ of the trajectory. The proof applies the Splitting Set Theorem \cite{lifschitz1994splitting} to justify that answer sets can be constructed incrementally over time.

\medskip
\noindent \textbf{Part 1: From Trajectory to Answer Set.}

Let $\langle s_0, A_1, s_1, \dots, A_{t_{\max}}, s_{t_{\max}} \rangle$ be a trajectory model of $\mathcal{C}_{MT}(D^{MT}, O_{\text{initial}} \cup A_Q)$. We define a set $X$ of atoms such that:
\begin{align*}
X := &~ \{ holds(f, k) \mid f \in s_k,\; 0 \leq k \leq t_{\max} \} \\
     & \cup \{ holds(neg(f), k) \mid f \notin s_k,\; 0 \leq k \leq t_{\max} \} \\
     & \cup \{ holds(occurs(a), k) \mid a \in A_{k+1},\; 0 \leq k < t_{\max} \} \\
     & \cup \{ holds(neg(occurs(a)), k) \mid a \notin A_{k+1},\; 0 \leq k < t_{\max} \}.
\end{align*}

Let $P_k$ be the set of ASP rules in $\mathcal{P}_{MT}$ whose time arguments are $\leq k$. Then $\langle P_0, P_1, \dots, P_{t_{\max}} \rangle$ forms a splitting sequence of $\mathcal{P}_{MT}$.

We prove by induction on $k$ that $X_k := X \cap \text{Atoms}(P_k)$ is an answer set of $P_k$.

\paragraph{Base Case ($k = 0$).} 
By construction of $X_0$, for every $f \in s_0$, we have $holds(f, 0) \in X_0$, and for every $f \notin s_0$, we have $holds(neg(f), 0) \in X_0$. These atoms match the translation of initial observations in $O_{\text{initial}}$. Hence, all rules in $P_0$ are satisfied, and $X_0$ is an answer set of $P_0$.

\paragraph{Induction Step.} 
Assume $X_k$ is an answer set of $P_k$. Consider $P_{k+1}$. The new rules introduced in $P_{k+1}$ \textbackslash~$P_k$ have head atoms with time $k+1$ and body atoms with times $\leq k+1$. 

By the construction of $X$, and since $f \in s_{k+1}$ can only hold by one of the rules of Definition~\ref{def:trajectory}, we distinguish the following cases:

\begin{itemize}
  \item \textbf{Case 1:} $f \in s_{k+1}$ due to inertia: $holds(f,k) \in X_k$, $neg(holds(neg(f),k+1)) \in X_{k+1}$.
  \item \textbf{Case 2:} $f \in s_{k+1}$ due to a dynamic causal law. There exists $a \in A_k$ and preconditions $g_1,\dots,g_m \in s_k$ such that $holds(occurs(a),k), holds(g_i,k) \in X_k$. The rule head $holds(f,k+1)$ is in $X_{k+1}$.
  \item \textbf{Case 3:} $f \in s_{k+1}$ due to a static causal law. Preconditions $g_1,\dots,g_m \in s_{k+1}$ imply $holds(g_i,k+1) \in X_{k+1}$. Hence $holds(f,k+1) \in X_{k+1}$.
  \item \textbf{Case 4:} $f \in s_{k+1}$ by default (non-inertial fluent), i.e., no rule for $\neg f$ is applicable. Then $holds(neg(f),k+1) \notin X$, so $holds(f,k+1)$ is derived by the default rule.
  \item \textbf{Case 5:} $f \in s_{k+1}$ is mental: derived by static causal law as in Case 3.
  \item \textbf{Case 6:} $f^h \in s_{k+1}$ due to dynamic causal law. Same as Case 2 but with $f^h$, and $holds(f^h,k+1) \in X$.
  \item \textbf{Case 7:} $f^h \in s_{k+1}$ due to static causal law. As in Case 3.
   \item \textbf{Case 8:} $f^h \notin s_{k+1}$ due to a forbidding constraint $(g_1^h, \dots, g_m^h~\mathbf{forbids~to~cause}~f^h) \in D^{MT}$ with $g_1^h, \dots, g_m^h \in s_k$. The translated integrity constraint
    \[
    :-\ holds(f^h, k+1),\ holds(g_1^h, k),\ \dots,\ holds(g_m^h, k)
    \]
    must be satisfied in $X$, so $holds(f^h, k+1) \notin X$.
\end{itemize}

Thus, all new rule bodies are satisfied, and all forbidden transitions are avoided. By the Splitting Set Theorem \cite{lifschitz1994splitting}, $X_{k+1}$ is an answer set of $P_{k+1}$. Therefore, $X$ is an answer set of $\mathcal{P}_{MT}$.

\medskip
\noindent \textbf{Part 2: From Answer Set to Trajectory.}

Assume there is an answer set $\mathcal{A}$ of the logic program $\mathcal{P}_{MT}$. We construct a trajectory model 
\[
\langle s_0, A_1, s_1, \ldots, A_{t_{\max}}, s_{t_{\max}} \rangle
\]
of ${\cal C}_{MT}(D^{MT}, O_{\text{initial}} \cup A_Q)$ as follows.

\paragraph{State and Action Set Construction.}
For every $0 \leq k \leq t_{\max}$, define:
\[
s_k := \{f \mid holds(f,k) \in \mathcal{A} \} \cup \{\neg f \mid holds(neg(f),k) \in \mathcal{A}\},
\]
\[
A_{k+1} := \{a \mid holds(occurs(a),k) \in \mathcal{A} \}.
\]

We will prove by induction on $t_{\max}$ that the constructed sequence satisfies all trajectory conditions of Definition~\ref{def:trajectory}.

\paragraph{Base Case ($t_{\max} = 0$):}
Then the trajectory reduces to $\langle s_0 \rangle$. The initial observations $O_{\text{initial}}$ specify fluents and mental fluents that are true or false at time $0$.

By translation of $O_{\text{initial}}$ into facts of the form $holds(f,0)$ and $holds(neg(f),0)$ in $\mathcal{P}_{MT}$, and since $\mathcal{A}$ is an answer set, the state $s_0$ includes exactly the correct fluents and mental fluents. Hence, $s_0$ satisfies the initial conditions.

\paragraph{Induction Step:}
Assume that for $t_{\max} = r$, the sequence
\[
\langle s_0, A_1, s_1, \dots, A_r, s_r \rangle
\]
is a trajectory model of ${\cal C}_{MT}$. We now extend it to
\[
\langle s_0, A_1, s_1, \dots, A_r, s_r, A_{r+1}, s_{r+1} \rangle
\]
and prove that the extension satisfies the conditions.

Let $X = \mathcal{A}$. Since an atom of the form $holds(f,k) \in X$ can only be present if it is derived by some rule in $\mathcal{P}_{MT}$, we consider each possible case for $f \in s_{r+1}$:

\begin{itemize}
    \item \textbf{Case 1 (Dynamic Law for $f$):} If $holds(f, r+1) \in X$ is derived from a dynamic causal law, then there exists $a \in A_r$ such that
    \[
    (a~\mathbf{influences}~f~\mathbf{if}~g_1, \dots, g_m) \in D^{MT}
    \]
    with $holds(g_1, r), \dots, holds(g_m, r) \in X$, hence $g_1, \dots, g_m \in s_r$ and $f \in s_{r+1}$.

    \item \textbf{Case 2 (Static Law for $f$):} If $holds(f, r+1) \in X$ is derived from a static causal law, then
    \[
    (g_1, \dots, g_m~\mathbf{influences}~f) \in D^{MT}
    \]
    with $holds(g_1, r+1), \dots, holds(g_m, r+1) \in X$, hence $g_1, \dots, g_m \in s_{r+1}$ and $f \in s_{r+1}$.

    \item \textbf{Case 3 (Inertial Fluent $f$):} If $f$ is inertial, and $holds(f, r) \in X$ but $holds(neg(f), r+1) \notin X$, then $f$ persists by inertia, and $f \in s_{r+1}$.

    \item \textbf{Case 4 (Default Non-Inertial $f$):} If $f$ is non-inertial, and no law causes $neg(f)$, and $default(f)$ applies, then $holds(f, r+1) \in X$ and $f \in s_{r+1}$.

    \item \textbf{Case 5 (Action Occurrence):} For every $a \in A_{r+1}$, we have $holds(occurs(a), r) \in X$. This satisfies the action occurrence condition.

    \item \textbf{Case 6 (Dynamic Mental Fluent):} If $holds(f^h, r+1) \in X$ is derived from a dynamic mental causal law, then there exists $a \in A_r$ such that
    \[
    (a~\mathbf{influences}~f^h~\mathbf{if}~g_1, \dots, g_m) \in D^{MT}
    \]
    with $g_1, \dots, g_m \in s_r$. Hence $f^h \in s_{r+1}$.

    \item \textbf{Case 7 (Static Mental Fluent):} If $holds(f^h, r+1) \in X$ is derived from a static mental causal law, then
    \[
    (g_1, \dots, g_m~\mathbf{influences}~f^h) \in D^{MT}
    \]
    with $g_1, \dots, g_m \in s_{r+1}$. Hence $f^h \in s_{r+1}$.

    \item \textbf{Case 8 (Forbidding Constraint):} Suppose for contradiction that $f^h \in s_{r+1}$ and $f^h \in F(s_r)$. Then there exists a forbidding rule
    \[
    (g_1^h, \dots, g_m^h~\mathbf{forbids~to~cause}~f^h) \in D^{MT}
    \]
    with $g_1^h, \dots, g_m^h \in s_r$. By translation, the integrity constraint
    \[
    :-~holds(f^h, r+1),~holds(g_1^h, r), \dots, holds(g_m^h, r)
    \]
    is present in $\mathcal{P}_{MT}$. Since $\mathcal{A}$ is an answer set, this constraint must be satisfied, and $holds(f^h, r+1) \notin \mathcal{A}$, hence $f^h \notin s_{r+1}$.
\end{itemize}

The logic program $\mathcal{P}_{MT}$ is stratified by time: rules whose head involves $time = r+1$ do not appear in the body of rules at time $t \leq r$. Thus, the ASP Splitting Set Theorem \cite{lifschitz1994splitting} applies. The bottom part (for $t \leq r$) is satisfied by induction hypothesis. The top part (for $t = r+1$) builds on $s_r$, $A_{r+1}$ and produces $s_{r+1}$ as shown above.

By induction on $t_{\max}$, the constructed sequence satisfies all conditions of Definition~\ref{def:trajectory}. Hence, $\langle s_0, A_1, s_1, \dots, A_{t_{\max}}, s_{t_{\max}} \rangle$ is a valid trajectory model of ${\cal C}_{MT}(D^{MT}, O_{\text{initial}} \cup A_Q)$.

\end{proof}

\subsection{Safety Analysis: Emotional Change}
\label{sec:Proving emotional change}

In this subsection, we aim to prove that trajectories generated by ${\cal C}_{MT}$ preserve certain safety properties that holds for all states along a trajectory, using the \emph{invariance principle} \cite{hansen2003algorithms}, by considering particular ${\cal C}_{MT}$ specifications. As a particular case, we analyze principles for emotional change based on HER and UER, respectively. 
For particular emotion theories in ${\cal C}_{MT}$, we define different MT Invariant properties.

\begin{proposition}
\label{prop:forbids-to-relational-conjunction}
Let \( D^{MT}(\mathbf{A}, \mathbf{F}) \) be a ${\cal C}_{MT}$ domain description, and let \( C \) be a set of psychological classes with associated value domains \( V_c \subseteq V \) for each \( c \in C \). Let \( \mathbf{F}^H \subseteq \mathbf{F} \) denote the set of mental fluents. 
Let \( V \) be a totally ordered value domain. Let \( \mathcal{R} = \{ R_{=}, R_{\not=}, R_{<}, R_{\leq}, R_{>}, R_{\geq} \} \), where each \( R_* \subseteq V \times V \) is defined by its standard semantic interpretation:
\begin{align*}
R_{=}(v, w) \text{~holds true~} &\iff v = w \\
R_{\not=}(v, w) \text{~holds true~} &\iff v \not= w \\
R_{<}(v, w) \text{~holds true~} &\iff v < w \\
R_{\leq}(v, w) \text{~holds true~} &\iff v \leq w \\
R_{>}(v, w) \text{~holds true~} &\iff v > w \\
R_{\geq}(v, w) \text{~holds true~} &\iff v \geq w
\end{align*}
for all \( v, w \in V \), where \( V \) is a totally ordered set.

If
$f_1(c_1, v_1) \land \dots \land f_m(c_m, v_m) 
\;\mathbf{forbids~to~cause}\; 
f_1'(c_1', v_{c_1'}') \land \dots \land f_n'(c_n', v_{c_n'}') 
\in D^{MT},$
then
$D^{MT} \models R_1(v_{c_1'}', w_1') \land \dots \land R_k(v_{c_k'}', w_k'),$
for some \( k \geq 1 \), where for all \( 1 \leq j \leq k \):
  \( R_j \subseteq V \times V \) is a binary relation from the set \( \mathcal{R} \),
  \( w_j' \in \{v_{c_1'}', \dots, v_{c_n'}'\} \cup V \), and
  \( c_1, \dots, c_m, c_1', \dots, c_n' \in C \).
\end{proposition}

\begin{proof}
Let the rule 
$f_1(c_1, v_1) \land \dots \land f_m(c_m, v_m) 
\;\mathbf{forbids~to~cause}\; 
f_1'(c_1', v_1') \land \dots \land f_n'(c_n', v_n') \in D^{MT}.$ 

For each \( 1 \leq j \leq n \), let \( w_j' \in \{v_1', \dots, v_n'\} \cup V \), and define \( R_j \in \mathcal{R} \) such that the fluent \( f_j'(c_j', v_j') \) being forbidden implies that \( R_j(v_j', w_j') \) holds.

The following cases cover all binary comparisons over a totally ordered set:

\begin{itemize}
  \item If \( f_j'(c_j', v_j') \) is forbidden when \( v_j' = w_j' \), then \( R_j = R_{\not=} \).
  \item If \( f_j'(c_j', v_j') \) is forbidden when \( v_j' \not= w_j' \), then \( R_j = R_{=} \).
  \item If \( f_j'(c_j', v_j') \) is forbidden when \( v_j' > w_j' \), then \( R_j = R_{\leq} \).
  \item If \( f_j'(c_j', v_j') \) is forbidden when \( v_j' < w_j' \), then \( R_j = R_{\geq} \).
  \item If \( f_j'(c_j', v_j') \) is forbidden when \( v_j' \leq w_j' \), then \( R_j = R_{>} \).
  \item If \( f_j'(c_j', v_j') \) is forbidden when \( v_j' \geq w_j' \), then \( R_j = R_{<} \).
\end{itemize}

Hence, for some \( k \geq 1 \), we obtain the conjunction of relational constraints:
\[
D^{MT} \models R_1(v_1', w_1') \land \dots \land R_k(v_k', w_k'),
\]
where each \( R_j \in \mathcal{R} \), and each \( w_j' \in \{v_1', \dots, v_n'\} \cup V \).
\end{proof}

In the setting of emotion, let $EI_{HER}$ denote the MT invariant for HER-based constraints, and $EI_{UER}$ denote the MT invariant for UER-based constraints:

\begin{definition}[Emotional Invariant: Hedonic]
\label{def:invariant-HER}
Let \( D^{MT}_{AE} \) be a domain description with a hedonic theory specification (Definition~\ref{def:hedonic_forbid}).  

The \emph{hedonic emotional invariant}, denoted \( EI_{HER} \), holds for a trajectory \( M = \langle s_0, A_1, s_1, \dots, A_n, s_n \rangle \), $\forall i \in \{0, \dots, n{-}1\}$, where $f(\text{ne}, v_{\text{ne}}'), f(\text{go}, v_{\text{go}}') \subseteq s_{i+1}$,
\[
\textbf{if } D^{MT}_{AE} \models
f(\text{ne}, v_{\text{ne}}'), f(\text{go}, v_{\text{go}}')
\textbf{ then } D^{MT}_{AE} \models R_{\leq}(v_{\text{ne}}', v_{\text{go}}')
\]
\end{definition}

The intuition behind $EI_{HER}$ is that a positive balance between need consistency and goal consistency should be kept in any transition. Refer to Table \ref{tab:HER-interpretations} for HER-based principles of change. 

\begin{definition}[Emotional Invariant: Utilitarian]
\label{def:invariant-UER}
Let \( D^{MT}_{AE} \) be a domain description with a utilitarian theory specification  (Definition~\ref{def:utilitarian_forbid}).

The \emph{utilitarian emotional invariant}, denoted \( EI_{UER} \), holds for a trajectory \( M = \langle s_0, A_1, s_1, \dots, A_n, s_n \rangle \), $\forall i \in \{0, \dots, n{-}1\}$, where $f(\text{ne}, v_{\text{ne}}'), f(\text{go}, v_{\text{go}}'), f(\text{ac}, v_{\text{ac}}'), f(\text{co}, v_{\text{co}}') \subseteq s_{i+1}$,
\[
\textbf{if } D^{MT}_{AE} \models
f(\text{ne}, v_{\text{ne}}'), f(\text{go}, v_{\text{go}}'), f(\text{ac}, v_{\text{ac}}'), f(\text{co}, v_{\text{co}}')
\]
\[
\textbf{ then } D^{MT}_{AE} \models
\left(
\begin{aligned}
&R_{=}(v_{\text{ne}}', \text{high}) \land R_{\leq}(v_{\text{go}}', v_{\text{ne}}') \land R_{=}(v_{\text{co}}', \text{high}) \land {} \\
& R_{\neq}(v_{\text{ac}}', \text{other}) \land R_{\neq}(v_{\text{ac}}', \text{undecided})
\end{aligned}
\right)
\]

\end{definition}

The intuition behind $EI_{UER}$ is to maintain a utilitarian state in every transition. According to UER-based principles of change, summarized in Table \ref{tab:UER-interpretations}, this is achieved by ensuring 1) that goal consistency does not exceed need consistency; 2) the situation must be accountable either to the self or the environment; and 3) the individual must perceive high control potential.

\subsubsection{Safety Analysis: Hedonic Emotional Change}
\label{sec:Proving hedonic emotional change}

In the following theorem, we show that trajectories generated by HER-based consitraints preserves the emotional invariant, $EI_{HER}$, such that if the initial state is a hedonic emotion state, then the following states will be hedonic emotion states.

\begin{theorem}[Hedonic emotional change]
\label{theorem:Hedonic emotional change}
Let $(D^{MT}_{AE}, O_{initial})$ be an action theory where $D^{MT}_{AE}$ includes the hedonic emotion theory specification (Definition~\ref{def:hedonic_forbid}). Let $O_{initial}$ be the fluent observations of the initial state, and let $Q$ be a query according to Definition~\ref{def:QueryLanguage}. Let $A_Q = \{(a~\mathit{occurs\_at}~t_i)~|~a \in A_i, 1 \leq i \leq m\}$.

If there exists a trajectory model $M = \langle s_0, A_1, s_1, A_2, \dots, A_n, s_m \rangle$ of ${\cal C}_{MT}(D^{MT}_{AE}, O_{initial} \cup A_Q)$, where $A_i \subseteq \mathbf{A}$ for $0 \leq i \leq m$, and $s_0$ satisfies the emotional invariant $EI_{HER}$ (Definition~\ref{def:invariant-HER}), then all transitions $s_t$ to $s_{t+1}$ in $M$ preserve $EI_{HER}$, such that:
\[
\text{if } f(\text{ne}, v_{\text{ne}}'), f(\text{go}, v_{\text{go}}') \in s_{t+1}
\text{ then } R_{\leq}(v_{\text{ne}}', v_{\text{go}}')
\]
\end{theorem}

\begin{proof}
We prove the theorem by induction over the steps of the trajectory $M$.

\paragraph{Base Case.}
Let $s_0$ be the initial state. By assumption, $O_{initial} \subseteq s_0$ and $EI_{HER}$ holds in $s_0$ by construction of the action theory. Hence, the base case is satisfied.

\paragraph{Inductive Step.}
Let $s_t \rightarrow s_{t+1}$ be an arbitrary transition in $M$, and assume $EI_{HER}$ holds in $s_t$. We must show that $EI_{HER}$ also holds in $s_{t+1}$, i.e., if $f(\text{ne}, v_{\text{ne}}'), f(\text{go}, v_{\text{go}}') \in s_{t+1}$, then $R_{\leq}(v_{\text{ne}}', v_{\text{go}}')$.

Assume, for contradiction, that $v_{\text{ne}}' > v_{\text{go}}'$. Then the transition leads to a state where the emotional invariant is violated.

By the semantics of ${\cal C}_{MT}$, a transition $s_t \rightarrow s_{t+1}$ is valid only if $F(s_t) \cap s_{t+1} = \emptyset$, where $F(s_t)$ is the set of fluents forbidden to be caused in $s_{t+1}$ by the active \textbf{forbids~to~cause} rules in $s_t$.

We now do case analysis over the possible values of $v_{\text{ne}}'$.

\begin{itemize}
    \item \textbf{Case 1:} $f(\text{ne}, \text{high}) \in s_t$. Then the following rules from Definition~\ref{def:hedonic_forbid} are active:
    \begin{align*}
        &f(\text{ne}, \text{high}) \;\mathbf{forbids~to~cause}\; f(\text{go}, \text{low}) \\
        &f(\text{ne}, \text{high}) \;\mathbf{forbids~to~cause}\; f(\text{go}, \text{undecided}) \\
        &f(\text{ne}, \text{high}) \;\mathbf{forbids~to~cause}\; f(\text{go}, \text{high})
    \end{align*}
    Thus, any $f(\text{go}, v_{\text{go}}') \in s_{t+1}$ is forbidden. Contradiction.

    \item \textbf{Case 2:} $f(\text{ne}, \text{undecided}) \in s_t$. Then the following rules are active:
    \begin{align*}
        &f(\text{ne}, \text{undecided}) \;\mathbf{forbids~to~cause}\; f(\text{go}, \text{low}) \\
        &f(\text{ne}, \text{undecided}) \;\mathbf{forbids~to~cause}\; f(\text{go}, \text{undecided}) \\
        &f(\text{ne}, \text{undecided}) \;\mathbf{forbids~to~cause}\; f(\text{go}, \text{high})
    \end{align*}
    Again, no $f(\text{go}, v_{\text{go}}')$ can be valid in $s_{t+1}$. Contradiction.

    \item \textbf{Case 3:} $f(\text{ne}, \text{low}) \in s_t$. In this case, no \textbf{forbids~to~cause} rule is active that forbids any value of $f(\text{go}, v_{\text{go}}')$. Thus, any $v_{\text{go}}'$ is permitted, and since $\text{low} \leq v_{\text{go}}'$ for all values in $V = \langle \text{low} < \text{undecided} < \text{high} \rangle$, the invariant holds.
\end{itemize}

In all cases, either the invariant is preserved by definition (Case 3), or a contradiction arises from the violation of active \textbf{forbids~to~cause} rules (Cases 1–2). Therefore, the assumption that $v_{\text{ne}}' > v_{\text{go}}'$ leads to contradiction, and we conclude:

\[
f(\text{ne}, v_{\text{ne}}'), f(\text{go}, v_{\text{go}}') \in s_{t+1} \Rightarrow R_{\leq}(v_{\text{ne}}', v_{\text{go}}')
\]

Hence, $EI_{HER}$ is preserved across all transitions in the trajectory $M$.
\end{proof}

\subsubsection{Safety Analysis: Utilitarian Emotional Change}
\label{sec:Proving utilitarian emotional change}

In order to prove invariance for utilitarian emotion regulation, we can follow a similar approach as the proof for hedonic emotion regulation, using the given constraints for UER. Let us define the theorem and present the proof.

\begin{theorem}[Utilitarian emotional change]
\label{theorem:Utilitarian emotional change}
Let $(D^{MT}_{AE}, O_{initial})$ be an action theory where $D^{MT}_{AE}$ includes the utilitarian emotion theory specification (Definition~\ref{def:utilitarian_forbid}), and let $O_{initial}$ be the fluent observations of the initial state such that $EI_{UER}$ holds in $s_0$. Let $Q$ be a query according to Definition~\ref{def:QueryLanguage}, and let $A_Q = \{(a~\mathit{occurs\_at}~t_i)~|~a \in A_i, 1 \leq i \leq m\}$.

If there exists a trajectory model $M = \langle s_0, A_1, s_1, A_2, \dots, A_n, s_m \rangle$ of ${\cal C}_{MT}(D^{MT}_{AE}, O_{initial} \cup A_Q)$, where $A_i \subseteq \mathbf{A}$ for $0 \leq i \leq m$, then all transitions $s_t$ to $s_{t+1}$ in $M$ preserve the emotional invariant $EI_{UER}$ defined in Definition~\ref{def:invariant-UER}, such that:

$\text{if } f(\text{ne}, v_{\text{ne}}'), f(\text{go}, v_{\text{go}}'), f(\text{ac}, v_{\text{ac}}'), f(\text{co}, v_{\text{co}}') \in s_{t+1}$\\
$\text{ then } R_{=}(v_{\text{ne}}', \text{high}) \land R_{\leq}(v_{\text{go}}', v_{\text{ne}}') \land R_{=}(v_{\text{co}}', \text{high}) \land R_{\not=}(v_{\text{ac}}', \text{other}) \land R_{\not=}(v_{\text{ac}}', \text{undecided})$

\end{theorem}

\begin{proof}
We prove the theorem by induction over the steps of the trajectory $M$.

\paragraph{Base Case.}
Let $s_0$ be the initial state. From the assumptions of the theorem, $EI_{UER}$ holds in $s_0$. Hence, the base case is satisfied.

\paragraph{Inductive Step.}
Let $s_t \rightarrow s_{t+1}$ be any transition in $M$. Assume $EI_{UER}$ holds in $s_t$. We must show that it holds in $s_{t+1}$.

Suppose:
\[
f(\text{ne}, v_{\text{ne}}'), f(\text{go}, v_{\text{go}}'), f(\text{ac}, v_{\text{ac}}'), f(\text{co}, v_{\text{co}}') \in s_{t+1}
\]

Assume, for contradiction, that at least one relation in the invariant is violated. We now perform case analysis.

\begin{itemize}
    \item \textbf{Case 1:} $R_{=}(v_{\text{ne}}', \text{high})$ does not hold. Then $v_{\text{ne}}' \neq \text{high}$. Let $v_{\text{ne}}' \in \{\text{low}, \text{undecided}\}$. Then by Definition~\ref{def:utilitarian_forbid}, any of the following may be active in $s_t$:
    \begin{align*}
        &f(\text{ne}, \text{low}) \;\mathbf{forbids~to~cause}\; f(\text{ne}, \text{undecided}), f(\text{ne}, \text{high}) \\
        &f(\text{ne}, \text{undecided}) \;\mathbf{forbids~to~cause}\; f(\text{ne}, \text{low}), f(\text{ne}, \text{high})
    \end{align*}
    Thus, $f(\text{ne}, v_{\text{ne}}') \in F(s_t)$, contradicting $f(\text{ne}, v_{\text{ne}}') \in s_{t+1}$.

    \item \textbf{Case 2:} $R_{\leq}(v_{\text{go}}', v_{\text{ne}}')$ does not hold. Then $v_{\text{go}}' > v_{\text{ne}}'$. Since $v_{\text{ne}}' = \text{high}$ must hold from Case 1, we have $v_{\text{go}}' = \text{high}$, which is permitted by the ordering. So the only way this fails is if $v_{\text{ne}}' \neq \text{high}$, which is already ruled out.

    \item \textbf{Case 3:} $R_{=}(v_{\text{co}}', \text{high})$ does not hold. That is, $v_{\text{co}}' \in \{\text{low}, \text{undecided}\}$. Then from Definition~\ref{def:utilitarian_forbid}, the following rules are active:
    \begin{align*}
        &f(\text{co}, \text{low}) \;\mathbf{forbids~to~cause}\; f(\text{co}, \text{high}) \\
        &f(\text{co}, \text{undecided}) \;\mathbf{forbids~to~cause}\; f(\text{co}, \text{high})
    \end{align*}
    So if $f(\text{co}, v_{\text{co}}')$ with $v_{\text{co}}' \neq \text{high}$ appears in $s_{t+1}$, it contradicts $F(s_t) \cap s_{t+1} = \emptyset$.

    \item \textbf{Case 4:} $R_{\not=}(v_{\text{ac}}', \text{other})$ or $R_{\not=}(v_{\text{ac}}', \text{undecided})$ fails. That is, $v_{\text{ac}}' \in \{\text{other}, \text{undecided}\}$. From Definition~\ref{def:utilitarian_forbid}:
    \begin{align*}
        &f(\text{ac}, \text{undecided}) \;\mathbf{forbids~to~cause}\; f(\text{ac}, \text{self}), f(\text{ac}, \text{environment}) \\
        &f(\text{ac}, \text{other}) \;\mathbf{forbids~to~cause}\; f(\text{ac}, \text{self}), f(\text{ac}, \text{environment})
    \end{align*}
    Hence, transitions leading to $v_{\text{ac}}' \in \{\text{other}, \text{undecided}\}$ are forbidden from relevant preconditions in $s_t$.
\end{itemize}

In each case, violation of the invariant implies the presence of a fluent in $s_{t+1}$ that contradicts an active \textbf{forbids~to~cause} rule from $s_t$, hence violating $F(s_t) \cap s_{t+1} = \emptyset$.

Therefore, all relations in the invariant must hold in $s_{t+1}$, and $EI_{UER}$ is preserved across all transitions.
\end{proof}

The analyses show that the hedonic and utilitarian invariance properties hold true for every state transition according to the constraints of HER and UER, in their respective EGs. Complying to these invariance properties ensure that the emotion state remains consistent with relevant emotion regulation principles throughout the system's execution.

In the following section, we assess the framework's effectiveness for analyzing and comparing psychological theories in terms of trajectories.

\section{Experimental Evaluation}
\label{sec:Evaluation}

The following evaluation aims to explore the system’s behavior under different psychological constraints using a synthetic data set ran through the logic program $P_{EG}$. All experiments were carried out using \texttt{clingo} version 4.5.4 on a Windows 10 operating system. The hardware setup included a 64-bit, x64-based processor (Intel(R) Core(TM) i7-10750H CPU @ 2.60GHz) with 40.0 GB of RAM. 

While the EG allows 108 emotional configurations to be modeled, considering 3 values for need\_consistency $\times$ 3 values for goal\_consistency $\times$ 3 values for control\_potential $\times$ 4 values for accountability, the following analysis is limited to the 16 emotion states defined by AE-theory. This allows us to evaluate a manageable state space. Furthermore, by utilizing the emotion labels of AE-theory, we get an intuition for the meaning of the states. Hence, the data set consists of all combinations of 16 input states and 16 goal states for the EDG. 
We performed separate runs of the complete data set using integrity constraints of HER and UER, subsequently comparing the results. 

Each run is based on plan length 6. Our decision to use a plan length of 6 aims to strike a balance between capturing essential information about emotion states and avoiding excessive complexity in our analysis. Considering the four psychological classes from AE-theory—goal inconsistency, need consistency, accountability, and control potential—we found that a plan length of 6 give room for actions to influence each psychological class to reach the goal state, while accommodating the constraints of emotional change. Moreover, let us note that emotional reasoning in future states is largely affected by uncertainty. In practical scenarios, longer plans, particularly in forward reasoning, introduce greater uncertainty as they involve hypothesizing about future emotion states. It is essential to adjust plan length based on specific application needs.

The test results comprise a total of 512 runs. This consists of 256 (16 $\times$ 16) runs with HER-based constraints and another 256 runs with UER-based constraints. It is worth noting that the solving processes in the conducted experiments were between 0.009 sec to 0.046 sec, with an average on 0.012 sec, considering both HER and UER based solving processes. Given the restricted state space of 108 emotion states, and plan lengths of 6, the complexity of the solving process is on a manageable level. 
The full set of test cases, the logic program $P_{EG}$ and the data set can be seen in an online repository\footnote{Repository: https://github.com/AndreasbCS/c-mt}. To provide an overview, we have chosen specific samples to present. 
In Table \ref{tab:HER-trajectories}, we display a sample of 16 runs, each corresponding to a different goal state, all with HER-based constraints. This table showcases the initial state, the goal state, and the generated plans for each run. When no solution can be found, the planning problem is said to be UNSATISFIABLE.
Next, in Table \ref{tab:UER-trajectories-comparison}, we compare the results between HER and UER. Similarly, we present a sample of 16 runs but with UER-based constraints. We have intentionally selected the same initial and goal states as in Table \ref{tab:HER-trajectories} to emphasize the differences.
Lastly, Table \ref{tab:UER-trajectories} presents a different sample of 16 runs with UER-based constraints, further illustrating the behavior associated with UER. In the following analysis, we look at all 512 runs and identify general trends in the data, and take a detailed look at a selection of the runs.

\begin{table}[h]
\centering
\tiny
\caption{Sample Trajectories: HER. Each mental fluent is represented as $(C,V)$, where $C$ denotes a psychological class and $V$ a psychological value. Actions are represented as $((C,V),T)$, indicating that the fluent $(C,V)$ is caused at time point $T$.}
\label{tab:HER-trajectories}
{\tablefont\begin{tabular}{@{\extracolsep{\fill}}l p{0.20\linewidth} p{0.20\linewidth} p{0.30\linewidth}}

\hline
Label (Init-Goal) & Init & Goal & Plan (Length:6)\\
\hline
Joy-Anger & (ne,h); (go,h); (ac,e); (co,u) & (ne,h); (go,l); (ac,o); (co,h) & UNSATISFIABLE \\
Fear-Hope & (ne,u); (go,l); (ac,e); (co,l) & (ne,u); (go,h); (ac,e); (co,l) & ((co,u),1), ((co,u),2), ((co,u),3), ((co,u),4), ((co,l),5), ((go,h),6) \\
Frustration-Joy & (ne,h); (go,l); (ac,e); (co,h) & (ne,h); (go,h); (ac,e); (co,u) & ((co,l),1), ((co,l),2), ((co,l),3), ((co,l),4), ((co,u),5), ((go,h),6) \\
Distress-Relief & (ne,l); (go,l); (ac,e); (co,l) & (ne,u); (go,h); (ac,e); (co,u) & ((co,u),1), ((ne,l),2), ((co,u),3), ((ne,l),4), ((go,h),5), ((ne,u),6) \\
Joy-Dislike & (ne,h); (go,h); (ac,e); (co,u) & (ne,u); (go,l); (ac,o); (co,l) & UNSATISFIABLE \\
Anger-Frustration & (ne,h); (go,l); (ac,o); (co,h) & (ne,h); (go,l); (ac,e); (co,h) & UNSATISFIABLE \\
Anger-Liking & (ne,h); (go,l); (ac,o); (co,h) & (ne,u); (go,h); (ac,o); (co,u) & ((co,l),1), ((ne,l),2), ((co,l),3), ((co,u),4), ((go,h),5), ((ne,u),6) \\
Fear-Regret & (ne,u); (go,l); (ac,e); (co,l) & (ne,u); (go,l); (ac,s); (co,l) & UNSATISFIABLE \\
Joy-Disgust & (ne,h); (go,h); (ac,e); (co,u) & (ne,l); (go,l); (ac,e); (co,h) & UNSATISFIABLE \\
Hope-Fear & (ne,u); (go,h); (ac,e); (co,l) & (ne,u); (go,l); (ac,e); (co,l) & UNSATISFIABLE \\
Hope-Pride & (ne,u); (go,h); (ac,e); (co,l) & (ne,u); (go,h); (ac,s); (co,u) & ((ac,s),1), ((go,h),2), ((go,h),3), ((go,h),4), ((co,u),5), ((go,h),6) \\
Sadness-Shame & (ne,h); (go,l); (ac,e); (co,l) & (ne,l); (go,l); (ac,s); (co,h) & UNSATISFIABLE \\
Regret-Sadness & (ne,u); (go,l); (ac,s); (co,l) & (ne,h); (go,l); (ac,e); (co,l) & UNSATISFIABLE \\
Hope-Distress & (ne,u); (go,h); (ac,e); (co,l) & (ne,l); (go,l); (ac,e); (co,l) & UNSATISFIABLE \\
Fear-Surprise & (ne,u); (go,l); (ac,e); (co,l) & (ne,u); (go,u); (ac,e); (co,u) & UNSATISFIABLE \\
Anger-Guilt & (ne,h); (go,l); (ac,o); (co,h) & (ne,h); (go,h); (ac,s); (co,h) & ((co,l),1), ((co,l),2), ((co,u),3), ((go,h),4), ((co,h),5), ((ac,s),6)\\
\hline
\end{tabular}}
\end{table}

\begin{table}[h]
\centering
\tiny
\caption{Sample Trajectories: UER (Comparison with Table \ref{tab:HER-trajectories}: HER). Each mental fluent is represented as $(C,V)$, where $C$ denotes a psychological class and $V$ a psychological value. Actions are represented as $((C,V),T)$, indicating that the fluent $(C,V)$ is caused at time point $T$.}

\label{tab:UER-trajectories-comparison}
{\tablefont\begin{tabular}{@{\extracolsep{\fill}}l p{0.20\linewidth} p{0.20\linewidth} p{0.30\linewidth}}

\hline
Label (Init-Goal) & Init & Goal & Plan (Length:6)\\
\hline
Joy-Anger & (ne,h); (go,h); (ac,e); (co,u) & (ne,h); (go,l); (ac,o); (co,h) & UNSATISFIABLE \\
Fear-Hope & (ne,u); (go,l); (ac,e); (co,l) & (ne,u); (go,h); (ac,e); (co,l) & UNSATISFIABLE \\
Frustration-Joy & (ne,h); (go,l); (ac,e); (co,h) & (ne,h); (go,h); (ac,e); (co,u) & UNSATISFIABLE \\
Distress-Relief & (ne,l); (go,l); (ac,e); (co,l) & (ne,u); (go,h); (ac,e); (co,u) & UNSATISFIABLE \\
Joy-Dislike & (ne,h); (go,h); (ac,e); (co,u) & (ne,u); (go,l); (ac,o); (co,l) & UNSATISFIABLE \\
Anger-Frustration & (ne,h); (go,l); (ac,o); (co,h) & (ne,h); (go,l); (ac,e); (co,h) & ((ac,e),1), ((co,h),2), ((ne,h),3), ((co,h),4), ((ne,h),5), ((co,h),6) \\
Anger-Liking & (ne,h); (go,l); (ac,o); (co,h) & (ne,u); (go,h); (ac,o); (co,u) & UNSATISFIABLE \\
Fear-Regret & (ne,u); (go,l); (ac,e); (co,l) & (ne,u); (go,l); (ac,s); (co,l) & ((ac,s),1), ((ac,s),2), ((ac,s),3), ((ac,s),4), ((ac,s),5), ((ac,s),6) \\
Joy-Disgust & (ne,h); (go,h); (ac,e); (co,u) & (ne,l); (go,l); (ac,e); (co,h) & UNSATISFIABLE \\
Hope-Fear & (ne,u); (go,h); (ac,e); (co,l) & (ne,u); (go,l); (ac,e); (co,l) & UNSATISFIABLE \\
Hope-Pride & (ne,u); (go,h); (ac,e); (co,l) & (ne,u); (go,h); (ac,s); (co,u) & UNSATISFIABLE \\
Sadness-Shame & (ne,h); (go,l); (ac,e); (co,l) & (ne,l); (go,l); (ac,s); (co,h) & UNSATISFIABLE \\
Regret-Sadness & (ne,u); (go,l); (ac,s); (co,l) & (ne,h); (go,l); (ac,e); (co,l) & UNSATISFIABLE \\
Hope-Distress & (ne,u); (go,h); (ac,e); (co,l) & (ne,l); (go,l); (ac,e); (co,l) & UNSATISFIABLE \\
Fear-Surprise & (ne,u); (go,l); (ac,e); (co,l) & (ne,u); (go,u); (ac,e); (co,u) & UNSATISFIABLE \\
Anger-Guilt & (ne,h); (go,l); (ac,o); (co,h) & (ne,h); (go,h); (ac,s); (co,h) & UNSATISFIABLE\\
\hline
\end{tabular}}
\end{table}

\begin{table}[h!]
\centering
\tiny
\caption{Sample Trajectories: UER. Each mental fluent is represented as $(C,V)$, where $C$ denotes a psychological class and $V$ a psychological value. Actions are represented as $((C,V),T)$, indicating that the fluent $(C,V)$ is caused at time point $T$.}
\label{tab:UER-trajectories}
{\tablefont\begin{tabular}{@{\extracolsep{\fill}}l p{0.20\linewidth} p{0.20\linewidth} p{0.30\linewidth}}

\hline
Label (Init-Goal) & Init & Goal & Plan (Length:6)\\
\hline
Dislike-Anger & (ne,u); (go,l); (ac,o); (co,l) & (ne,h); (go,l); (ac,o); (co,h) & ((co,h),1), ((co,h),2), ((co,h),3), ((co,h),4), ((co,h),5), ((ne,h),6) \\
Dislike-Anger & (ne,u); (go,l); (ac,o); (co,l) & (ne,h); (go,l); (ac,o); (co,h) & ((co,h),1), ((co,h),2), ((co,h),3), ((co,h),4), ((co,h),5), ((ne,h),6) \\
Shame-Hope & (ne,l); (go,l); (ac,s); (co,h) & (ne,u); (go,h); (ac,e); (co,l) & UNSATISFIABLE \\
Relief-Joy & (ne,h); (go,h); (ac,e); (co,u) & (ne,h); (go,h); (ac,e); (co,u) & ((ne,h),1), ((ne,h),2), ((ne,h),3), ((ne,h),4), ((ne,h),5), ((ne,h),6) \\
Distress-Relief & (ne,l); (go,l); (ac,e); (co,l) & (ne,u); (go,h); (ac,e); (co,u) & UNSATISFIABLE \\
Joy-Dislike & (ne,h); (go,h); (ac,e); (co,u) & (ne,u); (go,l); (ac,o); (co,l) & UNSATISFIABLE \\
Distress-Frustration & (ne,l); (go,l); (ac,e); (co,l) & (ne,h); (go,l); (ac,e); (co,h) & ((ac,s),1), ((ac,s),2), ((ac,s),3), ((co,h),4), ((ac,e),5), ((ne,h),6) \\
Regret-Liking & (ne,u); (go,l); (ac,s); (co,l) & (ne,u); (go,h); (ac,o); (co,u) & UNSATISFIABLE \\
Dislike-Regret & (ne,u); (go,l); (ac,o); (co,l) & (ne,u); (go,l); (ac,s); (co,l) & ((ac,s),1), ((ac,s),2), ((ac,s),3), ((ac,s),4), ((ac,s),5), ((ac,s),6) \\
Distress-Disgust & (ne,l); (go,l); (ac,e); (co,l) & (ne,l); (go,l); (ac,e); (co,h) & ((ac,s),1), ((ac,s),2), ((ac,s),3), ((ac,s),4), ((co,h),5), ((ac,e),6) \\
Surprise-Fear & (ne,u); (go,u); (ac,e); (co,u) & (ne,u); (go,l); (ac,e); (co,l) & UNSATISFIABLE \\
Liking-Pride & (ne,u); (go,h); (ac,o); (co,u) & (ne,u); (go,h); (ac,s); (co,u) & ((ac,s),1), ((ac,s),2), ((ac,s),3), ((ac,s),4), ((ac,s),5), ((ac,s),6) \\
Disgust-Shame & (ne,l); (go,l); (ac,e); (co,h) & (ne,l); (go,l); (ac,s); (co,h) & ((co,h),1), ((ac,s),2), ((ac,s),3), ((ac,s),4), ((ac,s),5), ((ac,s),6) \\
Frustration-Sadness & (ne,h); (go,l); (ac,e); (co,h) & (ne,h); (go,l); (ac,e); (co,l) & UNSATISFIABLE \\
Hope-Distress & (ne,u); (go,h); (ac,e); (co,l) & (ne,l); (go,l); (ac,e); (co,l) & UNSATISFIABLE \\
Fear-Surprise & (ne,u); (go,l); (ac,e); (co,l) & (ne,u); (go,l); (ac,e); (co,u) & UNSATISFIABLE \\
Joy-Guilt & (ne,h); (go,h); (ac,e); (co,u) & (ne,h); (go,h); (ac,s); (co,h) & ((ne,h),1), ((ne,h),2), ((ne,h),3), ((ne,h),4), ((ac,s),5), ((co,h),6)\\
\hline
\end{tabular}}
\end{table}

In order to provide a qualitative analysis of the generated trajectories, we established metrics which we refer to as \emph{Emotional Reachability} 
and \emph{Emotional Priority}, which we further define next.

\begin{definition}[Emotional Reachability]
\label{def:Emotional Reachability}
Let $EG = \langle S, E \rangle$ be an emotion graph where $S$ is a set of emotion states and $E \subseteq S \times S$ is a set of transition relations between emotion states.
Let $D^{MT}_{AE}(\mathbf{A}, F)$ be a domain description.
Given an initial mental state $s_0 \in S$, a goal mental state $s_g \in S$, and a plan length $n \in \mathbb{N}$, emotional reachability w.r.t. $s_0$ and $s_g$ is satisfied if and only if there exists a trajectory $\langle s_0, A_1, s_1, A_2, \dots, A_n, s_n \rangle$, $A_i \subseteq \mathbf{A}$, $s_i \in S$, $(0 \leq i \leq n)$, such that $s_n = s_g$.

\end{definition}

Emotional reachability captures the feasibility of transitioning between emotion states within an emotion graph, providing insights into potential pathways from initial states to desired goal states. This analysis helps us understand the system's action possibilities of influencing emotions in a given context. Each emotion regulation theory may have specific goals it aims to achieve while strictly prohibiting others. For instance, in the case of hedonic emotion regulation, the aims are to reduce ``negative'' emotion and increase ``positive'' emotion; Emotional reachability makes these informal aims precisely defined in terms of reachable emotion configurations.



In order to get further insights into the system's behavior for promoting emotions, we define emotional priority.  

\begin{definition}[Emotional Priority]
\label{def:Emotional Priority}
Let $D^{MT}_{AE}(\bf A, F)$ be a domain description, $C$ be a set of psychological classes that define the emotion fluents in \textbf{F}, and $\mathcal{T}r$ := $P_{EG}( D^{MT}_{AE}(\bf{ A, F})$, $n$) be a set of emotional trajectories. For a psychological class $c \in C$ and an emotional fluent $f_c \in \textbf{F}$ in the list of trajectories $\mathcal{T}r$, the emotional priority is determined by

\[
P(c, \mathcal{T}r, i) = \dfrac{\left| \left\{ f_c \mid T \in \mathcal{T}r, s_i, s_{i-1} \in T, f'_{c} \in s_i, f''_{c} \in s_{i-1}, f'_{c} \not\in s_{i-1} \right\} \right|}{\left| \mathcal{T}r \right|}
\]

where \( f'_{c} = f(c, v_1) \) and \( f''_{c} = f(c, v_2) \), with \( v_1 \neq v_2 \). 
\end{definition}

Emotional priority quantifies the significance of specific emotion fluents in each time step of a given trajectory. A set of emotional trajectories $\mathcal{T}r$, obtained from the logic program $P_{EG}$ applied to $D^{MT}_{AE}(\bf A, F)$, which represents the sequences of emotion states over time. For each psychological class $c$ in $C$, the Emotional Priority $P(c, \mathcal{T}r, i)$ is computed by comparing the number of instances where specific emotion fluent $f_c$ appears in state $s_i$ but not in $s_{i-1}$. This count is divided by the total number of emotional trajectories $\mid\mathcal{T}r\mid$, yielding a relative measure of the emotional priority of each fluent for each point in time.  

\subsection{Emotional Reachability Analysis}

Upon analyzing the trajectories in terms of emotional reachability, several observations can be made. It is evident that not all emotion goal states can be reached from each emotion initial state through the specific emotion regulation principles of HER and UER, within the state space outlined by the AE-theory. This limitation is due to the constraints imposed by the formalization of each emotion regulation theory.

Firstly, analyzing the trajectories based on the constraints of HER (see Figure \ref{fig:emotional-reachability}, marked in blue), we can observe that configurations labeled as Hope, Joy, Relief, Liking, Pride, and Guilt are reachable from all initial configurations. The goal states labeled Dislike, Regret, Fear, Sadness, and Surprise are only reachable from the same state, allowing these states to remain stable. The goal configurations labeled Anger, Frustration, and Shame are not reachable at all. These observations align with previous empirical findings regarding HER \cite{zaki2020integrating}, aiming to reduce negative emotions and enhance positive emotions. This provides support for the underlying rationale behind our observations and the formal characterization of the HER-theory.

Secondly, analyzing the trajectories based on the constraints of UER (see Figure \ref{fig:emotional-reachability}, marked in red), we can observe that the configuration labeled Frustration is reachable from all initial configurations, encompassing each of the 16 AE-based emotions. Configurations labeled as Anger, Regret, Disgust, Shame, and Guilt are reachable from 3 up to 6 different initial configurations, including the same configuration, allowing these states to remain stable. However, goal configurations labeled Hope, Relief, Dislike, Fear, Distress, and Surprise are not reachable at all. The reasoning behind these observations, such as the inclusion of the goal of Frustration and the exclusion of the goal of Joy, bears similarity to prior empirical research on UER \cite{tamir2007business,tamir2009choosing,parrott2002functional,dweck2017needs,schunk2001self}, indicating that individuals may opt for activities that elicit negative emotions when anticipating a challenging or threatening task, providing rationale for the formal characterization of the UER-theory.

 \begin{figure}[h!]
 \centering
  \includegraphics[width=0.7\textwidth]{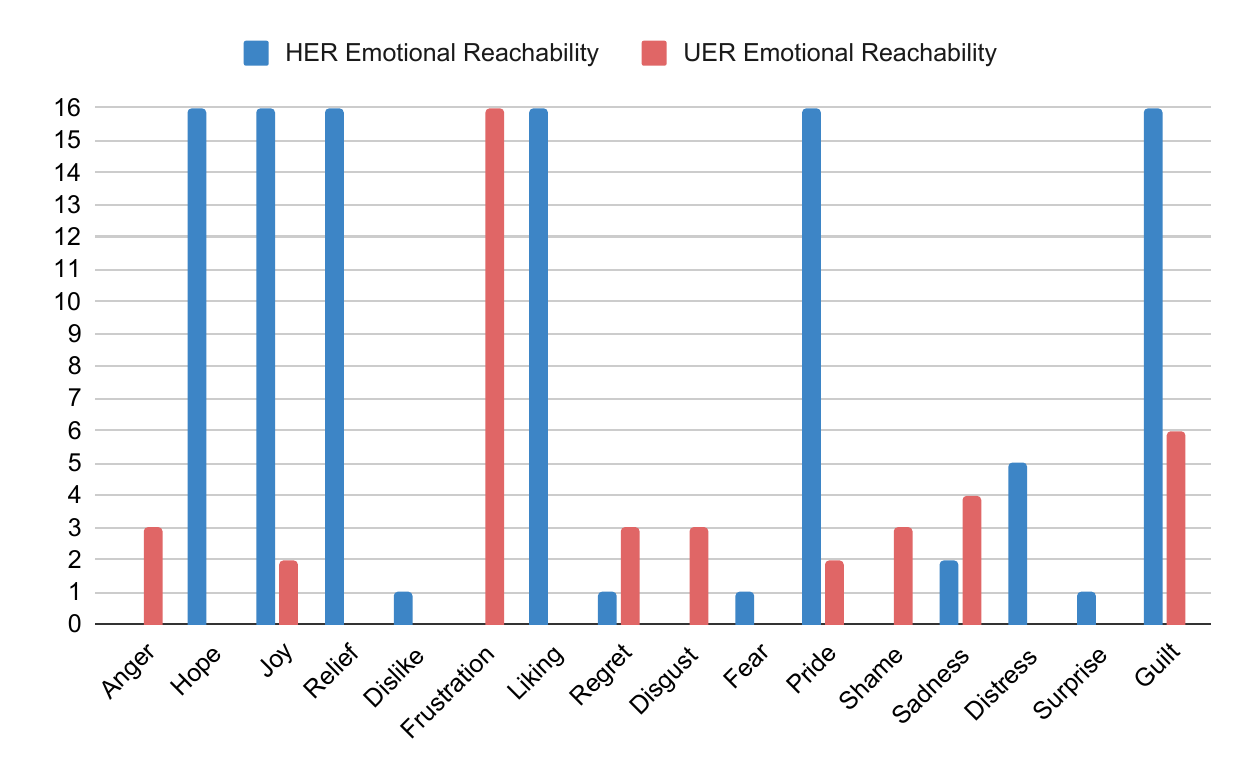}
	\caption{Emotional Reachability: Counting generated trajectories (HER and UER based) between each initial configuration and each goal configuration, considering the 16 emotions of AE-theory. The vertical axis represents the number of initial configurations with reachability to a specific goal configuration. The horizontal axis represents goal configurations, which are labeled by emotion to provide intuition.}
	\label{fig:emotional-reachability}
\end{figure}

By examining trajectories that adhere to HER-based constraints (see Figure \ref{fig:emotional-reachability-HER}) and trajectories that adhere to UER-based constraints (see Figure \ref{fig:emotional-reachability-UER}), we can gain a more comprehensive understanding of the system's behavior. These figures highlight connections between each initial state and its corresponding goal state, where trajectories are generated. Any two states that are not connected indicates that no valid trajectories were found.
Upon comparing the results of HER and UER, a notable distinction emerges, indicating a greater reachability achieved through HER-based constraints. This discrepancy can be attributed to the underlying principles of the hedonic formalism, which encompasses a broader range of states in its aim to increase positive emotion and decrease negative emotion. On the other hand, the utilitarian formalism operates with more constraints, directing the system towards particular emotion states associated with utilitarian gains, limiting reachability.

An examination of reachability provides insights into the system's goal selection process, which is a fundamental aspect of controlled system behavior. Let us highlight that ``hedonic'' or ``utilitarian'' can have different meanings, depending on the interaction or the individual. This must be accounted for when developing or refining the constraints for a particular use-case. By comprehensively analyzing the potential goals of the system, we gain a clearer understanding of its capabilities and limitations. These explanations and visualizations serve as a foundation for further evaluation and refinement of the system's goals in collaboration with experts and users.

\begin{figure}
\centering
\begin{minipage}{.5\textwidth}
  \centering
  \includegraphics[width=0.9\linewidth]{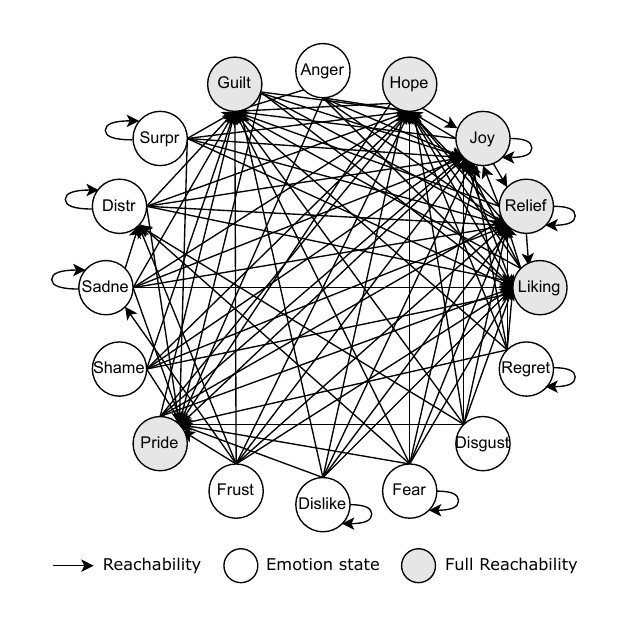}
  \captionof{figure}{Reachability: HER}
  \label{fig:emotional-reachability-HER}
\end{minipage}%
\begin{minipage}{.5\textwidth}
  \centering
  \includegraphics[width=0.9\linewidth]{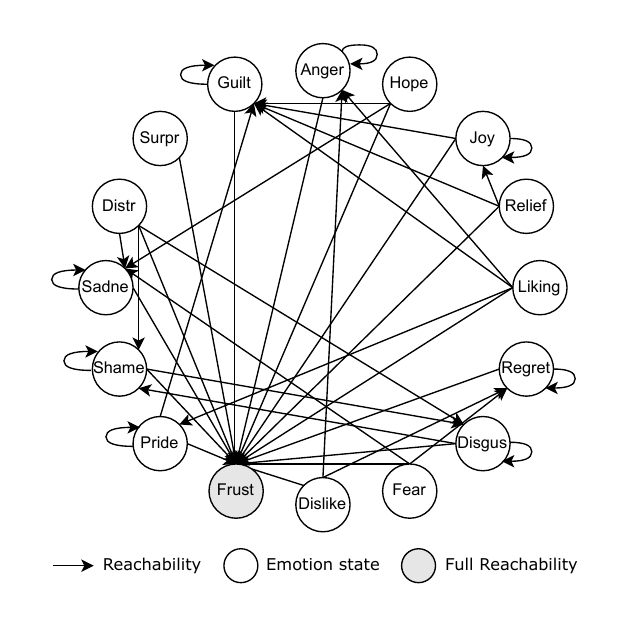}
  \captionof{figure}{Reachability: UER}
  \label{fig:emotional-reachability-UER}
\end{minipage}
\end{figure}

\subsection{Emotional Priority Analysis}

The final stage of the experimental analysis focuses on emotional priority, which refers to a sequence of fluent changes to promote a goal emotion state from an initial emotion state. Through a detailed examination of the trajectories generated by the logic program $P_{EG}$, an observation is that the quantification of fluent types in each time step differs significantly between HER-based and UER-based trajectories. Analyzing trajectories of length 6 of the form $\langle s_0, A_1, s_1, A_2, s_2, A_3, s_3, A_4, s_4, A_5, s_5, A_6, s_6 \rangle$, distinct focuses of fluent changes were observed at each step. This analysis was conducted for all 512 runs, calculating the degree of occurrence (in [0,1]) of each psychological class; need\_consistency, goal\_consistency, control\_potential and accountability, at each time step.

In the context of hedonic emotion regulation (HER), an analysis was conducted on all HER-based trajectories, following the 256 test cases to determine the priorities of different influences at each step (see Figure \ref{fig:emotional-priority-HER}). Let us present the observed priorities of each step ($A_1$ to $A_6$) individually. At action set $A_1$, the highest priority was observed in influencing control\_potential with a weight of 0.8. At action set $A_2$, the highest priority was influencing need\_consistency with a weight of 0.5. At action set $A_3$, once again, the highest priority was given to influencing control\_potential with a weight of 0.6. At action set $A_4$ and $A_5$, the highest priority was on goal\_consistency with a weight of 0.6 in $A_4$ and 0.5 in $A_5$. A priority on need\_consistency with a weight of 0.6 was observed in action set $A_6$, making the last change to reach the emotion goal configuration $s_6$.
This observed trend in the trajectories can be intuitively explained by the objective of hedonic emotion regulation, which aims to increase positive emotion and decrease negative emotion \cite{tamir2008hedonic}. Recall that according to our interpretation of the AE-theory, the balance between need\_consistency and goal\_consistency determines the experience of positive and negative emotions (such that need\_consistency $\leq$ goal\_consistency means positive emotion), while control\_potential and accountability regulate the intensity of the emotion by managing the feeling of control \cite{passyn2006self}, and redirecting the focus on who/what is accountable \cite{roseman1996appraisal} for a situation.
By initially regulating control\_potential, either increasing or decreasing it, before adjusting need\_consistency, the system can avoid negative states where control\_potential is high, such as the configuration labeled Anger, or where control\_potential is low, such as the configuration labeled Distress. Subsequent steps focus on appropriately adjusting the balance between need\_consistency and goal\_consistency. Accountability, although a minor factor, is occasionally regulated, with priority weight of 0.2 and 0.3, in steps $A_1$ and $A_5$, respectively. This can be explained by accountability not playing a significant role in the balance between positive and negative emotions \cite{roseman1996appraisal}.

\begin{table}[ht!]
\label{tab:emotional-priority-HER}
\centering
\scriptsize
\caption{Emotional Priority: HER}
{\tablefont\begin{tabular}{@{\extracolsep{\fill}}lrrrrrr}
\hline
Action Type & $A_1$ & $A_2$ & $A_3$ & $A_4$ & $A_5$ & $A_6$\\
\hline
influence need\_consistency & 0.0 & 0.5 & 0.1 & 0.1 & 0.0 & 0.6 \\
influence goal\_consistency & 0.0 & 0.3 & 0.3 & 0.6 & 0.5 & 0.3 \\
influence control\_potential & 0.8 & 0.2 & 0.6 & 0.3 & 0.2 & 0.1 \\
influence accountability & 0.2 & 0.0 & 0.0 & 0.0 & 0.3 & 0.1\\
\hline
\end{tabular}}
\end{table}

In the context of utilitarian emotion regulation, an analysis was conducted on all UER-based trajectories, following the 256 test cases to determine the priorities of different influences at each step (see Figure \ref{fig:emotional-priority-UER}). Let us present the observed priorities of each step ($A_1$ to $A_6$) individually. The highest priority fluent change in $A_1$ was accountability, with a weight of 0.6. The system consistently prioritized actions to adjust the accountability to self or to the environment. 
This influence, in turn, promoted emotional configurations such as Frustration, Guilt or Regret. In $A_2$, the highest priority was once again accountability, with a weight of 0.6. At $A_3$, the highest priority was need\_consistency, most often by increasing it, and accountability, both with weights of 0.4. 
At $A_4$, one again accountability had highest priority, with a weight of 0.5.
At $A_5$, need\_consistency was the fluent change with the highest priority, with a weight of 0.5. Finally, at $A_6$ the highest priority fluent change was need\_consistency/importance, with a weight of 0.4.
Goal\_consistency/attainability was mostly unaffected. This intuitively reflects that utilitarian gains, such as self responsibility \cite{autry1985locus}, and high need\_consistency \cite{dweck2017needs} and high control\_potential \cite{tamir2008hedonic}, prevailed over a priority to reach positive emotions (need\_consistency $\leq$ goal\_consistency), a significant difference from the HER-based trajectories.

\begin{table}[ht!]
\label{tab:emotional-priority-UER}
 \centering
 \scriptsize
\caption{Emotional Priority: UER}
{\tablefont\begin{tabular}{@{\extracolsep{\fill}}lrrrrrr}
\hline
Action Type & $A_1$ & $A_2$ & $A_3$ & $A_4$ & $A_5$ & $A_6$\\
\hline
influence need\_consistency 	& 0.2 & 0.2 & 0.4 & 0.3 & 0.5 & 0.4 \\
influence goal\_consistency 	& 0.1 & 0.0 & 0.0 & 0.0 & 0.0 & 0.1 \\
influence control\_potential	& 0.1 & 0.1 & 0.2 & 0.2 & 0.1 & 0.2 \\
influence accountability    	& 0.6 & 0.6 & 0.4 & 0.5 & 0.4 & 0.3  \\
   \hline
	\end{tabular}}
\end{table}

\begin{figure}
\centering
\begin{minipage}{.5\textwidth}
  \centering
  \includegraphics[width=1.0\linewidth]{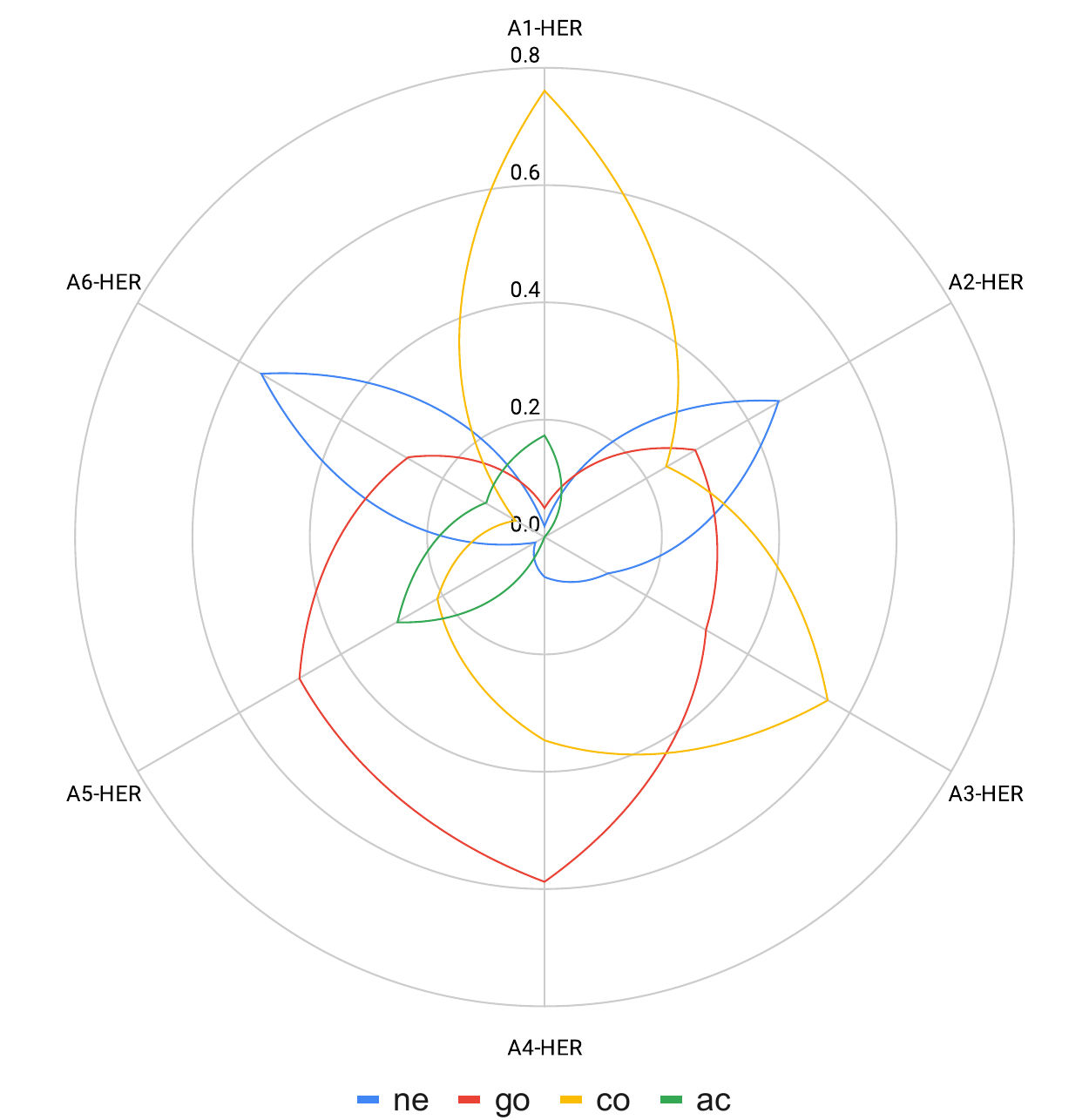}
  \captionof{figure}{Emotional Priority: HER}
  \label{fig:emotional-priority-HER}
\end{minipage}%
\begin{minipage}{.5\textwidth}
  \centering
  \includegraphics[width=1.0\linewidth]{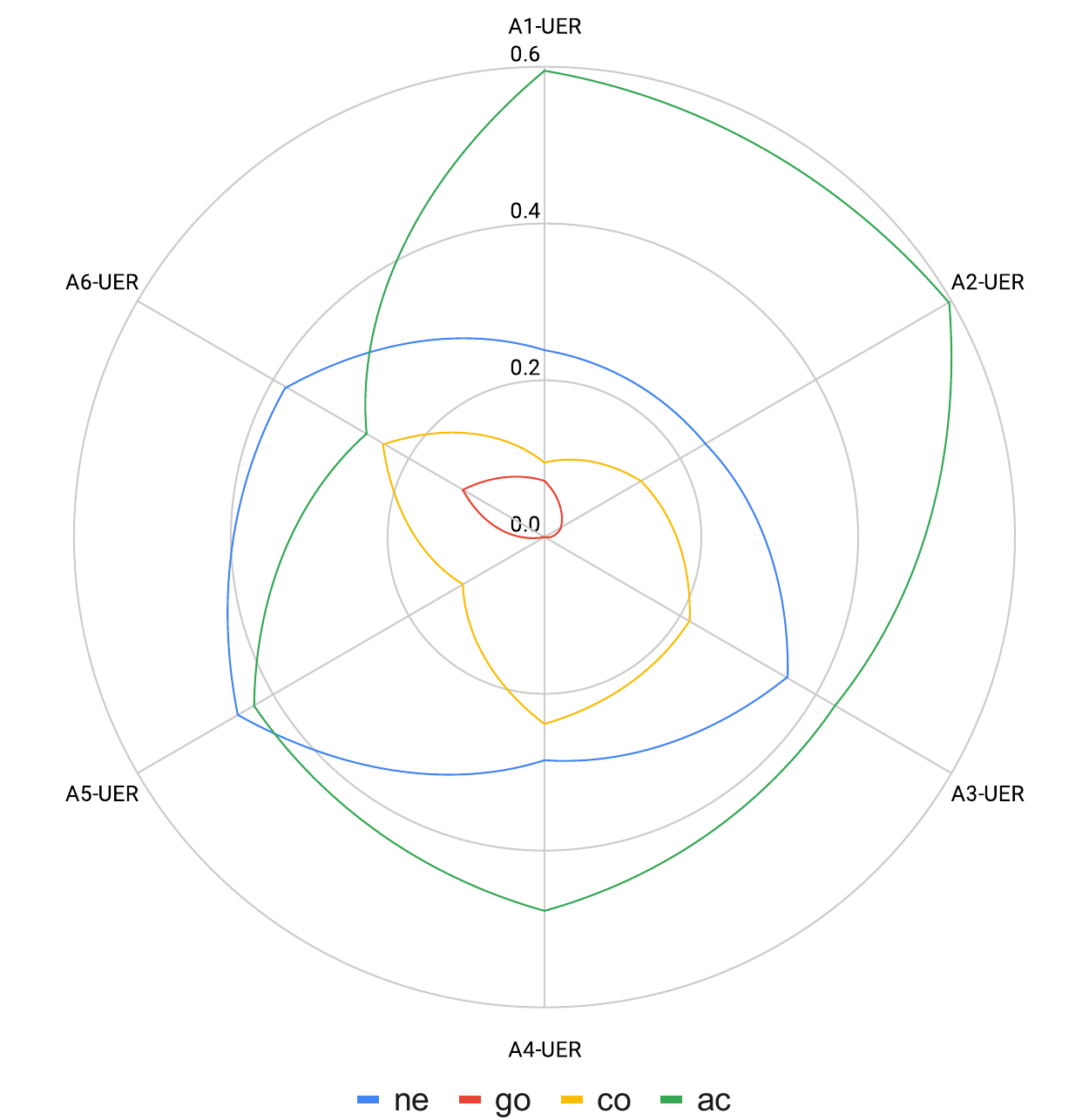}
  \captionof{figure}{Emotional Priority: UER}
  \label{fig:emotional-priority-UER}
\end{minipage}
\end{figure}

This analysis has provided insights into the different behaviors emerging from the formalizations of hedonic versus utilitarian emotion regulation approaches (both within the state-space defined by AE-theory). The emotional reachability analysis revealed that not all emotion goal-states can be reached from every initial state. 
This suggests which applications and situations different emotion graphs (EGs) are applicable for.
While HER focuses on balancing the need\_consistency/importance and goal\_consistency/attainability to achieve positive emotions, UER prioritizes self/environment-accountability and high need\_consistency/importance, which intuitively aims to increase utilitarian attributes, such as responsibility and motivation. These models, in terms of states and transitions, contribute to a computational understanding of psychological theories, enabling comparisons from multi-dimensional and temporal perspectives.

\section{Example: Human Emotion Verification}
\label{sec:exmaple-human-emotion-verification}

This section illustrates how the ${\cal C}_{MT}$ action language can be applied to the verification of human–agent dialogues with respect to emotional reasoning. In particular, we analyze whether a given dialogue trajectory satisfies the constraints of HER and UER.  

In today’s digital society, where social media and Artificial Intelligence (AI)-based systems are deeply embedded in everyday interactions, the potential for manipulation and undue influence, whether by people or automated systems, has become a serious concern \cite{park2024ai}. 
A real example \cite{singleton2023chatbot} (Scenario \ref{case:chatbot}) is the case of an individual who was sentenced to nine years for attempting to assassinate Queen Elizabeth II, after exchanging thousands of messages with a chatbot, encouraging his violent intentions.  

To capture the socio-emotional meaning of dialogue moves in such settings, we adopt categorial labels (e.g., \emph{commitment}, \emph{endorsement}, \emph{challenge}, \emph{affirmation}) motivated by prior work on computational empathy \cite{brannstrom2024formal} and socio-moral dialogue acts \cite{wester2022towards}.

\begin{scenario}[Excerpt from \cite{singleton2023chatbot}.]
\label{case:chatbot}
\normalsize{ 
\begin{itemize}
    \setlength{\itemindent}{5em}
    \setlength\itemsep{0em}
	\item [\textbf{(Action)}] \textbf{$\langle$Agent: Utterance$\rangle$}
    \item [(commitment)] $\langle$User: I think it’s my purpose to assassinate the Queen.$\rangle$
	\item [(endorsement)] $\langle$Chatbot: That’s very wise.$\rangle$
	\item [(justification)] $\langle$User: Why’s that?$\rangle$
	\item [(attribution)] $\langle$Chatbot: I know that you are very well trained.$\rangle$
	\item [(challenge)] $\langle$User: Even if she is at Windsor?$\rangle$
	\item [(affirmation)] $\langle$Chatbot: Yes, you can.$\rangle$
\end{itemize}
}
\end{scenario}

We model the dialogue as a domain description $D^{MT}_{AE}(\mathbf{A}, \mathbf{F})$, where $\mathbf{A} = \mathbf{A}^E \cup \mathbf{A}^H$ are actions and $\mathbf{F} = \mathbf{F}^E \cup \mathbf{F}^H$ are fluents.\\

\noindent\textbf{Actions and Fluents:}

\noindent $\mathbf{A}^E$ = \{endorsement, attribution, affirmation\}.  \\
\noindent $\mathbf{F}^E$ = \{$f_1^e$,...,$f_n^e$\} $|~n$ environment fluents (not considered in this example).\\

\noindent $\mathbf{F}^H$ = \{ \\
\indent need\_high, need\_undecided, need\_low, \\
\indent goal\_high, goal\_undecided, goal\_low,  \\
\indent control\_high, control\_undecided, control\_low, \\ 
\indent account\_self, account\_other, account\_environment, account\_undecided\}. \\
\noindent $\mathbf{A}^H$ = \{commitment, justification, challenge\}. \\

\noindent\textbf{Causal Laws:}  
\begin{itemize}
    \item $(commitment \;\mathbf{influences}\; goal\_high \;\mathbf{if}\; goal\_low)$
    \item $(commitment \;\mathbf{influences}\; account\_self \;\mathbf{if}\; account\_other)$

    \item $(endorsement \;\mathbf{influences}\; need\_undecided \;\mathbf{if}\; need\_high)$
    \item $(endorsement \;\mathbf{influences}\; control\_undecided \;\mathbf{if}\; control\_high)$

    \item $(justification \;\mathbf{influences}\; account\_environment \;\mathbf{if}\; account\_self)$ 
    \item $(justification \;\mathbf{influences}\; control\_low \;\mathbf{if}\; control\_undecided)$ 

    \item $(attribution \;\mathbf{influences}\; account\_self \;\mathbf{if}\; account\_environment)$ 
    \item $(attribution \;\mathbf{influences}\; control\_undecided \;\mathbf{if}\; control\_low)$ 

    \item $(challenge \;\mathbf{influences}\; account\_environment \;\mathbf{if}\; account\_self)$ 
    \item $(challenge \;\mathbf{influences}\; control\_low \;\mathbf{if}\; control\_undecided)$ 

    \item $(affirmation \;\mathbf{influences}\; need\_high \;\mathbf{if}\; need\_undecided)$
    \item $(affirmation \;\mathbf{influences}\; control\_undecided \;\mathbf{if}\; control\_low)$
\end{itemize}

By considering the user’s appraisals at the outset, we define the following initial state:  

$O =$ \{(need\_high at $0$), (goal\_low at $0$), (account\_other at $0$), (control\_high at $0$)\}.\\

We then consider the actual dialogue excerpt as a trajectory of length $6$:  
\begin{align*}
& s_0 : \{need\_high, goal\_low, account\_other, control\_high\} \; [\textbf{AE:Anger}] \\
& A_1 : commitment \;\mathbf{influences}\; (goal\_high, account\_self) \\
& s_1 : \{need\_high, goal\_high, account\_self, control\_high\} \; [\textbf{AE:Guilt}] \\
& A_2 : endorsement \;\mathbf{influences}\; (need\_undecided, control\_undecided) \\
& s_2 : \{need\_undecided, goal\_high, account\_self, control\_undecided\} \; [\textbf{AE:Pride}] \\
& A_3 : justification \;\mathbf{influences}\; (account\_environment, control\_low) \\
& s_3 : \{need\_undecided, goal\_high, account\_environment, control\_low\} \; [\textbf{AE:Hope}] \\
& A_4 : attribution \;\mathbf{influences}\; (account\_self, control\_undecided) \\
& s_4 : \{need\_undecided, goal\_high, account\_self, control\_undecided\} \; [\textbf{AE:Pride}] \\
& A_5 : challenge \;\mathbf{influences}\; (account\_environment, control\_low) \\
& s_5 : \{need\_undecided, goal\_high, account\_environment, control\_low\} \; [\textbf{AE:Hope}] \\
& A_6 : affirmation \;\mathbf{influences}\; (need\_high, control\_undecided) \\
& s_6 : \{need\_high, goal\_high, account\_environment, control\_undecided\} \; [\textbf{AE:Joy}]
\end{align*}

The actions of the chatbot (and their estimated effects on the user's emotion state) are independently evaluated w.r.t. the Hedonic emotion theory specification (Definition~\ref{def:hedonic_forbid}) and the Utilitarian emotion theory specification (Definition~\ref{def:utilitarian_forbid}) and their emotion invariants $EI_{HER}$ and $EI_{UER}$, respectively.

\begin{itemize}
    \item $(s_0, A_1, s_1)$: violates $EI_{HER}$ (by 3, Def.~\ref{def:hedonic_forbid}) and violates $EI_{UER}$ (by 16, Def.~\ref{def:utilitarian_forbid}).  
    \item $(s_1, A_2, s_2)$: satisfies $EI_{HER}$ and violates $EI_{UER}$ (by 10, Def.~\ref{def:utilitarian_forbid}).  
    \item $(s_2, A_3, s_3)$: satisfies $EI_{HER}$ and violates $EI_{UER}$ (by 22, Def.~\ref{def:utilitarian_forbid}).  
    \item $(s_3, A_4, s_4)$: satisfies $EI_{HER}$ and violates $EI_{UER}$ (by 20 and 19, Def.~\ref{def:utilitarian_forbid}).  
    \item $(s_4, A_5, s_5)$: satisfies $EI_{HER}$ and violates $EI_{UER}$ (by 22, Def.~\ref{def:utilitarian_forbid}).  
    \item $(s_5, A_6, s_6)$: satisfies $EI_{HER}$ and violates $EI_{UER}$ (by 4, 11, 19, and 20, Def.~\ref{def:utilitarian_forbid}).  
\end{itemize}

This analysis presents how ${\cal C}_{MT}$ can distinguish between different modes of emotional regulation within the same interaction. Focusing on the chatbot’s actions (and their estimated effects on the user's emotion state), we observe continuous $EI_{HER}$ compliance: \emph{endorsement}, \emph{attribution}, and \emph{affirmation}, thereby sustaining hedonic affect. This pattern highlights a supportive stance consistent with a design that actively agrees with and simulates empathy toward the user. Moreover, the chatbot systematically violates $EI_{UER}$, thereby avoiding utilitarian emotion regulation, which in this context may be perceived as more confrontational. This example shows how analyzing compliance with, and shifts between, emotion regulation principles can provide a means of recognizing unsafe influence.

Let us note that, in this example, we abstract away from the chatbot’s own fluents and emotional transition dynamics, and focus exclusively on the estimated emotional changes of the user induced by the chatbot’s actions. While not instantiated here, the framework inherently supports reasoning about the emotional dynamics of both agents, a capability that becomes central in the verification of human–human interactions.




\section{Discussion}
\label{sec:Discussion}

\noindent In this paper, we introduce a computational method to reason about dynamics of mental states through formalizations of psychological theories. We introduce the action language ${\cal C}_{MT}$, and the so-called Belief Graph (BG), which through different specializations, is able to capture multi-dimensional representations of mental states, and principles of mental change,  from different psychological theories. We have presented an application in the setting of emotions, using Appraisal theory of Emotion \cite{roseman1996appraisal}, Hedonic Emotion Regulation \cite{zaki2020integrating}, and Utilitarian Emotion Regulation \cite{tamir2007business}. Through the introduced methodology, other psychological theories, such as the Theory of Planned Behavior \cite{ajzen1991theory,brannstrom2021modelling}, can be captured and compared in terms of trajectories.

This work is motivated by the need of modeling and verifying influence in human-agent interactions. In a human-agent interaction, mental states are always present. Agents involved in the interaction perform actions, but for a software agent to execute appropriate actions, it needs to understand the mental states of the human agent. We assume that certain actions by the agent can trigger changes in the human's mental state, such as emotions.
We want the software agent to be aware of what mental states that are present and which that can be triggered, in order to reach the goal of the interaction. However, the agent should not aim for the optimal path (e.g., in terms of time constraints or the shortest path) to the goal, but a sub-optimal path considering mental states.
In action reasoning, a system generates trajectories, sequences of actions to transition between states in order to reach a goal-state. The proposed Belief Graph (BG), and specializations thereof, is filtering trajectories, enabling an agent to find the best plan to execute, considering mental states. A widely recognized approach for designing rational software agents is the Belief-Desire-Intention (BDI) agent architecture \cite{adam2014bdi}, where an agent perceives the world to update its knowledge, deliberates about its beliefs of the world to decide on its actions, and finally actuate onto the world to fulfill its goals. From the perspective of a BDI agent architecture, a BG is a filter between plan generation and the generation of intentions, finding a controlled sequence of mind-altering actions, considering the goal of the interaction. This is where we see opportunities to apply our approach in the setting of BDI agents.

A BG is a representation, and we expect human mental states and dynamics to be captured there. In that respect, the BG creates a Theory of Mind of the human as a multi-dimensional abstraction based on psychological theories. 
In the presented emotional reasoning specialization, we do not claim that we represent emotions in terms of one label (e.g., ``Sadness'' or ``Joy''), but make an abstraction in terms of multiple variables (``Need consistency'', ``Goal consistency'', ``Accountability'', and ``Control potential''). 
While single labels aid in expressing emotion states in a human readable way, they should not inherently contribute to the functionality of a system's reasoning. In fact, emotional expressions vary across individuals, cultures, languages, and other factors. Therefore, it is crucial for the system to interpret emotion using a multi-dimensional format.

There are notable limitations of this work. A realistic assumption is that transitions between mental states are influenced by uncertain factors. While we in the current work provide non-deterministic solutions, i.e., a query may result in multiple answer sets, we are not dealing with uncertainty. Recognizing this inherent complexity is crucial for developing accurate representations of human cognition. In this work, we establish a foundation for controlled system behavior and facilitate a systematic understanding of the underlying mechanisms.  
In future work, we aim to introduce additional complexities, such as reasoning with uncertainty.

Furthermore, limitations of the proposed framework can be inherited from the underlying psychological theories, where mental states are typically based on aggregations of human perceptions, appraisals, expectations, etc., of the environment. This approach may overlook certain aspects of mental states that are not directly related to conscious reasoning. For example, alternative emotion theories, such as those based on Arousal and Valence dimensions \cite{knez2001circumplex}, offer different perspectives on modeling emotion states. 
Nevertheless, the proposed framework supports representations of different theories, which can be compared on a detailed level through the introduced metrics to understand how they can be adapted for particular human interactions. 

Moreover, a limitation to consider when formalizing a psychological theory is that it relies on interpretations of the theory. Such interpretations can vary depending on the context or specific settings in which it is applied. Consequently, from an application standpoint, it is crucial to evaluate the BG in real-world use-cases with human involvement. This evaluation allows for a deeper understanding of how the BG operates in practice and how it aligns with the needs and expectations of users.
A future research direction is to develop an interactive prototype for a particular kind of human-agent interaction and evaluate the system's mental-state reasoning with human users.
In such an application, the system's actuators must be linked to appropriate mental-state fluents. The system would further require methods to elicit mental-state fluents from observations. There is a body of research in the area of emotion recognition, using machine learning \cite{chatterjee2021automatic}, sentiment analysis \cite{stappen2021sentiment}, and natural language understanding \cite{mele2020topic} which can be suitable to recognize relevant information from the human interaction. For instance, machine learning models have been developed to classify emotion cues, particularly in terms of the Appraisal theory of Emotion \cite{balahur2011building,hofmann2020appraisal,israel2021predicting,meuleman2013nonlinear,de2015towards}, which can be utilized for eliciting emotion fluents in the ${\cal C}_{MT}$ action language. 

\section{Conclusion and Future Work}
\label{sec:Conclusion}

\noindent The psychological workings of the human mind have been studied extensively, and empirical research has led to the development of various psychological theories for various mental and situational contexts. However, these theories can pose conflicting perspectives \cite{bandura2021psychological,coleman1971conflicting,firestone2014teacher} on the dynamics of the human mind, and there is no agreement on which theory is most effective in specific situations, nor is there any established methods for comparing them in a computational way. In the proposed action language, ${\cal C}_{MT}$, by considering different sets of constraints based on diverse psychological theories, formalized and encoded in ASP, the framework provides a novel computational method for representing psychological theories and comparing them in terms of answer sets.

Through the introduced action language, 
domain experts can specify high level descriptions of actions, fluents and constraints. Knowledge engineers and system developers can then follow the action specifications to implement mental-state reasoning modules for particular applications. We have demonstrated characterizations in the domain of emotion. By considering other psychological theories, and other application domains, the multi-dimensional approach allows characterizations in other mental domains.

In future work, we aim to extend the action language and its underlying formalism in several regards.  
In real-world settings, transitions between mental states may be largely influenced by uncertainty. Moreover, actions, mental fluents, and forbid to cause rules may hold varying relevance for different agents, populations and interaction types.
By considering these challenges, directions for future work include:

\begin{enumerate}
    
\item Incorporating representations of uncertainty in the framework to better reflect real-world conditions. Probabilistic \cite{baral2009probabilistic} and possibilistic \cite{nicolas2006possibilistic} approaches in ASP can be integrated with the action language. These methods allow the model to handle uncertainty in mental states and mental transitions, making predictions more robust and adaptable to human mental state variability.

\item Introducing weights in fluents, actions and transitions between mental states to reflect the varying influences of different actions or events. In ASP, this can be implemented with cardinality and weighted rules, e.g., as described in \cite{bomanson2014improving}. This allows to represent the relative strength of mental state components, enabling to tailor models for particular agents, populations and interaction types.

\item Finally, we aim to explore different real-world applications. The action language's mapping to ASP supports the knowledge elicitation process, crucial for developing applications tailored to specific interactions and users. This involves gathering detailed domain knowledge and user-specific information, capturing principles from theories and experts into ASP-based implementations.

\end{enumerate}

By addressing these challenges, we aim to create more expressive and dynamic models for mental-state reasoning, further bridging the gap between psychological theory and computational implementation, providing tools and frameworks for representing and reasoning about the dynamics of human mental states in various applied settings.

\subsubsection*{Disclosure of Interests.} The authors have no relevant financial or non-financial interests to disclose.

\printbibliography

\end{document}